\documentclass[english,twoside,11pt]{article}

\usepackage{jmlr2e}

\usepackage{pgffor}

\usepackage[english]{babel}
\usepackage{verbatim}
\usepackage{mathtools}
\usepackage{amsmath}
\usepackage{amsthm}
\usepackage{esint}
\usepackage{microtype}
\usepackage[capitalise,nosort]{cleveref}

\usepackage{subcaption}
\usepackage{caption}

\usepackage{pgf}

\crefformat{equation}{(#2#1#3)}
\crefname{prop}{Proposition}{Propositions}
\crefname{thm}{Theorem}{Theorems}
\crefname{sec}{Section}{Sections}
\crefname{figure}{Figure}{Figures}
\numberwithin{equation}{section}
\numberwithin{figure}{section}
\numberwithin{table}{section}
\theoremstyle{plain}
\newtheorem{thm}{\protect\theoremname}[section]
\theoremstyle{plain}
\newtheorem{cor}[thm]{\protect\corollaryname}
\theoremstyle{plain}
\newtheorem{prop}[thm]{\protect\propositionname}
\theoremstyle{plain}
\newtheorem{lem}[thm]{\protect\lemmaname}
\theoremstyle{plain}
\newtheorem{defn}[thm]{\protect\defnname}

\hypersetup{final}

\providecommand{\corollaryname}{Corollary}
\providecommand{\lemmaname}{Lemma}
\providecommand{\propositionname}{Proposition}
\providecommand{\theoremname}{Theorem}
\providecommand{\defnname}{Definition}

\newcommand\e{\mathrm{e}}%
\newcommand\R{\mathbf{R}}%
\newcommand\Z{\mathbf{Z}}%
\providecommand\C{}%
\renewcommand\C{\mathbf{C}}%
\newcommand\Rd{{\mathbf{R}^d}}%
\newcommand\1{\mathbf{1}}%
\newcommand\N{\mathbf{N}}%
\newcommand\sgn{\operatorname{sgn}}%
\newcommand\supp{\operatorname{supp}}%
\newcommand\diam{\operatorname{diam}}%
\newcommand\esssup{\operatorname*{ess\,sup}}%
\newcommand\essinf{\operatorname*{ess\,inf}}%
\renewcommand{\subset}{\subseteq}
\newcommand{\Ll}{\left}
\newcommand{\Rr}{\right}
\newcommand\dif{\mathrm{d}}%
\newcommand\mcl{\mathcal}
\newcommand\eps{\varepsilon}%
\newcommand\ep{\varepsilon}%
\renewcommand{\d}{\mathrm{d}}
\newcommand{\ii}{\mathrm{i}}
\renewcommand{\tilde}{\widetilde}
\newcommand{\td}{\widetilde}

\usepackage{lastpage}
\jmlrheading{25}{2024}{1-\pageref{LastPage}}{5/21; Revised
3/24}{4/24}{21-0495}{Alexander Dunlap and Jean-Christophe Mourrat}

\ShortHeadings{SON clustering does not separate nearby balls}{Dunlap and Mourrat}
\firstpageno{1}

\begin{document}%

\title{Sum-of-norms clustering does not separate nearby balls}

\author{\name Alexander Dunlap \email dunlap@math.duke.edu \\
        \addr Courant Institute of Mathematical Sciences \\
        New York University\\
        New York, NY 10012, USA
        \\\emph{Current address:} Duke University, Durham, NC 27708, USA\\
        \AND
        \name Jean-Christophe Mourrat \email jean-christophe.mourrat@ens-lyon.fr \\
        \addr Ecole Normale Sup\'erieure de Lyon and CNRS\\
Lyon,
France,\\and Courant Institute of Mathematical Sciences \\
        New York University\\
        New York, NY 10012, USA}

\editor{Ingo Steinwart}

\maketitle

\begin{abstract}
	Sum-of-norms clustering is a popular convexification of $K$-means clustering. We show that, if the dataset is made of a large number of independent random variables distributed according to the uniform measure on the union of two disjoint balls of unit radius, and if the balls are sufficiently close to one another, then sum-of-norms clustering will typically fail to recover the decomposition of the dataset into two clusters. As the dimension tends to infinity, this happens even when the distance between the centers of the two balls is taken to be as large as $2\sqrt{2}$. In order to show this, we introduce and analyze a continuous version of sum-of-norms clustering, where the dataset is replaced by a general measure. In particular, we state and prove a local-global characterization of the
	clustering that seems to be new even in the case of discrete datapoints.
\end{abstract}


\maketitle

\section{Introduction}

\subsection{Sum-of-norms clustering}
Clustering is the task of partitioning a dataset with the aim to optimize a measure of similarity between objects in each element of the partition. Given datapoints $x_1, \ldots, x_N \in \Rd$, one may seek to find $K$ ``centers'' so as to minimize the sum of the distances between each datapoint and its nearest center. This is the $K$-means problem, which can be formulated as follows: find $y_1,\ldots, y_N \in \Rd$ that minimize
\begin{equation*}  %
	\sum_{n = 1}^N |y_n - x_n|^2,
\end{equation*}
subject to the constraint that the set $\{y_1,\ldots, y_N\}$ has cardinality $K$ (or at most $K$). Here and throughout, $|\cdot|$ denotes the Euclidean norm. However, the $K$-means problem is NP-hard in general, even when we restrict to $K = 2$ \citep{aloise09} or to $d = 2$ \citep{MNV09}. In this article, we focus on a particular convex relaxation of $K$-means, introduced by \citet{PDBSDM05,HVBJ11,LOL11} and called ``convex clustering shrinkage,'' ``clusterpath,'' or ``sum-of-norms (SON) clustering,'' which consists in finding the points $y_1,\ldots, y_N \in \Rd$ that minimize
\begin{equation}
	\label{e.discrete.J}
	\frac{1}{N} \sum_{n = 1}^N |y_n - x_n|^2  + \frac{\lambda}{N^2} \sum_{k,n = 1}^N |y_k - y_n|,
\end{equation}
where $\lambda \ge 0$ is a tunable parameter. Two datapoints $x_k$ and $x_n$ are then declared to belong to the same cluster if $y_k = y_n$. In principle, varying the parameter $\lambda$ allows one to tune the number of clusters, as illustrated in \cref{fig:lambda-effect-explanation}.
\begin{figure}
    \begin{subfigure}[t]{2.9in}
        \centering
        \input{plots/three-circles-1.1.tex}
        \caption{$\lambda=1.1$. Each point is in its own cluster.}
    \end{subfigure}
    \hfill
    \begin{subfigure}[t]{2.9in}
        \centering
        \input{plots/three-circles-2.4.tex}
        \caption{$\lambda = 2.4$. The points in the upper circle have merged into a single cluster, but each point in the lower two circles remains in its own cluster.}
    \end{subfigure}
    \hfill
    \begin{subfigure}[t]{2.9in}
        \centering
        \input{plots/three-circles-3.4.tex}
        \caption{$\lambda = 3.4$. Each of the circles now forms a single cluster.}
    \end{subfigure}
    \hfill
    \begin{subfigure}[t]{2.9in}
        \centering
        \input{plots/three-circles-3.6.tex}
        \caption{$\lambda = 3.6$. There is now a single cluster comprising all of the points.}
    \end{subfigure}
    \caption{The output of the clustering algorithm on $N=100$ datapoints divided between the boundaries of three balls, for four values of $\lambda$. The filled circles represent the datapoints $x_n$, and the crosses represent the cluster representatives $y_n$. Each color represents a cluster. All figures in this paper were generated using an implementation (by the present authors) of the algorithm described in \cite{JV20}. The code is available at \url{https://github.com/ajdunlap/son-clustering-experiments}. %
    \label{fig:lambda-effect-explanation}}
\end{figure}One of the attractive features of SON clustering is that it produces an ordered path of partitions as we vary $\lambda$. In other words, its natural output is a hierarchy of nested partitions of the dataset \citep[see][or \cref{thm:agglomeration} below]{HVBJ11, CGR17}.

In the last decade, rigorous guarantees on the behavior of SON clustering have been studied by several authors, including \cite{ZXLY14,TW15,CGR17,PDJB17,RM17,JVZ19,CS19,JV20,STY21, NM21}. Most of these works aim at the identification of sufficient conditions for SON clustering to succeed in separating clusters. Our main goal here, stated precisely in Theorem~\ref{t.stoch.ball}, is rather to present a seemingly simple clustering problem in which the SON clustering algorithm will typically fail. This requires us to establish necessary \emph{and} sufficient conditions for the success of SON clustering, which we present in Subsection~\ref{subsec:clusterstructure}. We anticipate that these conditions will be useful in future studies of sum-of-norms clustering, and thus are interesting results in their own right.

Most of our attention will be towards the analysis of the following generalization of SON clustering: given a nonzero finite Borel measure $\mu$ on $\Rd$ of compact support and $\lambda \ge 0$, we seek to minimize the functional $J_{\mu,\lambda}\colon L^{2}(\mu;\mathbf{R}^{d})\to\mathbf{R}$
given by
\begin{equation}
	\label{e.def.J}
	J_{\mu,\lambda}(u) \coloneqq \int|u(x)-x|^{2}\,\dif\mu(x)+\lambda\iint|u(x)-u(y)|\,\dif\mu(x)\,\dif\mu(y).
\end{equation}

As will be explained at the beginning of \cref{sec:KKT}, the functional $J_{\mu,\lambda}$ has a unique minimizer, which we denote by $u_{\mu,\lambda} \in L^2(\mu; \Rd)$. The level sets of $u_{\mu,\lambda}$ yield
a partition of\textbf{ $\mathbf{R}^{d}$}, up to modifications by $\mu$-null
sets. One of the main general results of our paper, which seems to be new even in the discrete setting, is a local-global characterization of this minimizer, see \cref{thm:splitcondition} below. The correspondence between \eqref{e.discrete.J} and \eqref{e.def.J} is obtained by setting $\mu = \frac 1 N \sum_{n = 1}^N \delta_{x_n}$ and $y_n = u(x_n)$.

\subsection{The stochastic ball model}
The main motivation for introducing the continuous version of SON clustering is that it allows us to uncover the asymptotic behavior of the discrete problem in \eqref{e.discrete.J} when the number of datapoints $N$ becomes very large. In particular, we will study the ``stochastic ball model,'' which has become a common testbed in the analysis of clustering algorithms, see for instance \cite{NW15, awasthi2015relax, IMPV17, LLLS20, de2020ratio}. That is, we suppose that we are given a large number of points sampled independently at random, each being distributed according to the uniform measure on the union of two disjoint balls of unit radius, and ask whether SON clustering allows us to identify the presence of the two balls. Surprisingly, we find that if $d \ge 2$ and the balls are too close to each other, then the algorithm will typically fail to do so.

In order to state this result more precisely, we need to introduce some notation.
We write
\begin{equation}
	\label{e.def.gamma}
	\gamma_d \coloneqq \frac{2d+1}{2d+4} \cdot
	\begin{cases}
		\frac{(d+1)(2d)!\pi}{2^{3d}((d/2)!)^2 d!}      & \mbox{ if  $d$ is even}, \\
		\frac{(d+1)(( (d-1)/2 )!)^2(2d)!}{2^{d}(d!)^3} & \mbox{ if  $d$ is odd},
	\end{cases}
\end{equation}
so that
\begin{equation}  %
    \gamma_1 = 1, \qquad \gamma_2 = \frac{45\pi}{128} \simeq 1.104\ldots, \qquad \gamma_3 =  \frac 7 6,\label{eq:gammas}
\end{equation}
and
\begin{equation*}  %
	\frac{\gamma_{d+2}}{\gamma_d} = 1 + \frac{7d + 13}{(d+1)(2d+4)(2d+8)} > 1.
\end{equation*}
In particular, for every $d \ge 2$, we have $\gamma_d > 1$, and using Stirling's approximation, one can check that~$\gamma_d$ tends to~$\sqrt{2}$ as $d$ tends to infinity. We also write $B_r(x)$ for the open Euclidean ball or radius $r> 0$ centered at~$x \in \R^d$, and $(\e_1,\ldots,\e_d)$ for the canonical basis of $\R^d$. We use the phrase ``with high probability'' as shorthand for ``with probability tending to $1$ as $N$ tends to infinity''.
\begin{thm}
	\label{t.stoch.ball}
	There exists a $\lambda_{\mathrm{c}} \in (0,\infty)$ such that the following holds. Let $r \in [1,\gamma_d)$, $\mu$ be the uniform probability measure on $B_1(-r\e_1) \cup B_1(r\e_1) \subset \R^d$, $(X_n)_{n \in \N}$ be independent random variables with law $\mu$, and for every integer $N \ge 1$, define the empirical measure
	\begin{equation}
		\label{e.def.mun}
		\mu_N \coloneqq \frac 1 N \sum_{n = 1}^N \delta_{X_n}.
	\end{equation}
	\begin{enumerate}
		\item\label{part1}  If $\lambda > \lambda_{\mathrm{c}}$, then with high probability, the range of $u_{\mu_N,\lambda}$ is a singleton.

		\item\label{part2} If $\lambda < \lambda_{\mathrm{c}}$, then there exist $\xi, \eta > 0$ (not depending on $N$) such that, with high probability, one can find $A_N^{(1)}, A_N^{(2)}, A_N^{(3)} \subset \{1,\ldots, N\}$, each of cardinality at least $\xi N$ and satisfying, for every $i \neq j \in \{1,2,3\}$,
		      \begin{equation*}  %
			      \forall k \in A_N^{(i)}, \ \forall \ell \in A_N^{(j)}, \quad
			      \Ll|u_{\mu_N,\lambda}(X_k) - u_{\mu_N,\lambda}(X_\ell) \Rr| \ge \eta.
		      \end{equation*}
		      In particular, with high probability, the range of $u_{\mu_N,\lambda}$ contains at least three points.
	\end{enumerate}
    In fact, we can take $\lambda_{\mathrm{c}} = \lambda_1(\mu)$, with the latter quantity defined in \cref{eq:lambda1def} below.
\end{thm}
\begin{figure}
    \hfill
    \begin{subfigure}[t]{2.9in}
        \centering
        \input{plots/shattered-balls.tex}
        \caption{$\lambda=2.0$}
    \end{subfigure}
    \hfill
    \begin{subfigure}[t]{2.9in}
        \centering
        \input{plots/cohesive-balls.tex}
        \caption{$\lambda=2.15$}
    \end{subfigure}
    \hfill
    \caption{Sum-of-norms clustering of the stochastic ball model with $N=200$ datapoints drawn from $B(-1.05\e_1,1)\cup B(1.05\e_1,1)$. The balls from which the points are drawn are outlined in dotted grey lines. When $\lambda =2.0$, there are many clusters, but when $\lambda$ is slightly larger ($\lambda = 2.15$), there is just one large cluster. \cref{t.stoch.ball} tells us that (since $1.05<\gamma_2$), in the limit as $N\to\infty$, there will be \emph{no} open interval of values of $\lambda$ for which there are exactly two clusters. \label{fig:stochastic-ball-model-simple}}
\end{figure}
\cref{t.stoch.ball} does not describe the behavior of $u_{\mu_N,\lambda}$ when $\lambda = \lambda_{\mathrm{c}}$, or when $\lambda = \lambda_{\mathrm{c}}+o(1)$ as $N\to\infty$. But at the very least, \cref{t.stoch.ball} shows that the detection of two nearby balls by means of SON clustering will be particularly brittle. In contrast, we show in \cref{p.separation} that, using the notation of \cref{t.stoch.ball}, if $r > 2^{1-\frac 1 d}$ and $\lambda \in (2^{2-\frac 1 d}, 2r)$, then with high probability, the level sets of $u_{\mu_N,\lambda}$ are the sets $\{X_n, n \le N\} \cap B_1(-r\e_1)$ and $\{X_n, n \le N\} \cap B_1(r\e_1)$.

In a nutshell, SON clustering fails to separate balls if $r < \gamma_d$, while it succeeds if $r > 2^{1-\frac 1 d}$; see \cref{fig:stochastic-ball-model-simple} for an illustration of this failure when $r < \gamma_d$. We expect neither of these two bounds to be sharp. In view of \cref{cor:dsphere} and of the fact that points in a high-dimensional ball tend to concentrate near the boundary, we conjecture that in the limit of high dimensions, the threshold separating these two regimes converges to~$\sqrt{2}$. Since $\lim_{d \to \infty}\gamma_d = \sqrt{2}$, this would indicate that the lower bound on this threshold provided by \cref{t.stoch.ball} is asymptotically sharp.

\cref{t.stoch.ball} demonstrates in particular that the cardinality of the partition produced by the SON clustering algorithm can be very sensitive to small changes in the parameter $\lambda$. While \cref{t.stoch.ball} only asserts that the cardinality of the partition quickly moves from $1$ to at least~$3$ as we only slightly vary $\lambda$, we expect that the partition quickly shatters into many more than just three pieces. This is also what we observe in simulations, see \cref{fig:stochastic-ball-model-simple}. We view this phenomenon as a possible theoretical confirmation of the empirical observations of \cite{CGR17} and \cite{NM21}. We refer in particular to Figure~1(b) of \cite{CGR17} and the general observation that the tree structures produced by the (unweighted) SON clustering algorithm are often difficult to interpret (``unbalanced''), since the root of the tree very quickly splits into way too many components. (\citealp{CGR17}, also underline that among these many components, some will be much larger than others.) See also Figure~4 of \cite{NM21}.

\subsection{The structure of clusters\label{subsec:clusterstructure}}
\cref{t.stoch.ball} will be proved as a consequence of more general structural results on the clusters obtained by the sum-of-norms clustering algorithm. We foresee these results being useful in more general circumstances as well, and proceed to describe them now.

There are two special cases of clustering that will be particularly
important in our discussion.  We record them in the following definition.
\begin{defn} Let $\mu$ be a finite Borel measure of compact support and $\lambda\ge0$.
	\begin{enumerate}
		\item We say that $\mu$ is \emph{$\lambda$-cohesive} if there is a constant $c$ such that $u_{\mu,\lambda}\equiv c$, $\mu$-a.e.
		\item We say that $\mu$ is \emph{$\lambda$-shattered} if there is a measurable injection $u\colon\R^d\to\R^d$ such that $u_{\mu,\lambda}=u$, $\mu$-a.e.
	\end{enumerate}
\end{defn}
Note that if $\supp\mu$ consists of a single point  (or if $\mu$ is the zero measure), then $\mu$ is both $\lambda$-shattered and $\lambda$-cohesive for all $\lambda\ge 0$.

Recall that the level sets of~$u_{\mu,\lambda}$ define a partition of $\Rd$ up to  a $\mu$-null modification.
We think of this partition as a clustering of the support of $\mu$. To discuss these clusters, we will often use the notation
\[
	V_{u,x}\coloneqq u^{-1}(u(x))
\]
for the cluster containing $x$. The set $V_{u,x}$ is a Borel subset of $\R^d$ defined up to a $\mu$-null modification. Thus, saying that $\mu$ is $\lambda$-cohesive is equivalent to saying that $V_{u_{\mu,\lambda},x}=\R^d$ (up to a $\mu$-null modification) for $\mu$-a.e.~$x\in\R^d$. If $\mu$ is  $\lambda$-shattered, then $\mu(V_{u_{\mu,\lambda},x}\setminus\{x\})=0$ for $\mu$-a.e.~$x\in\R^d$, and in fact, by \cref{prop:regular} below, the converse holds as well.

The following theorem, proved in \cref{sec:splitcondition}, extends to the continuous setting results proved in the discrete case by \citet{CGR17}; see also Theorem~1 of \citet{JVZ19}.
\begin{thm}
	\label{thm:chiquet-characterization}
	For $\mu$-a.e.~$x\in\R^d$, the measure $\mu|_{V_{u_{\mu,\lambda},x}}$ is $\lambda$-cohesive, and if $A\ni x$ is such that $\mu|_A$ is $\lambda$-cohesive, then $\mu(A\setminus V_{u_{\mu,\lambda},x})=0$.
\end{thm}
It is not difficult to see, directly from \cref{e.def.J}, that if $\mu$ is $\lambda$-cohesive,
then it is also $\lambda'$-cohesive for any $\lambda'\ge\lambda$. As explained in more details in \cref{sec:splitcondition}, \cref{thm:chiquet-characterization} therefore implies the following theorem, referred to in the literature as the \emph{agglomeration conjecture} of \citet{HVBJ11}, and also proved in the discrete case by \citet{CGR17}.
\begin{thm}
	\label{thm:agglomeration}If $\lambda\le \lambda'$ then for $\mu$-a.e.\ $x$ we have $\mu(V_{u_{\mu,\lambda},x}\setminus V_{u_{\mu,\lambda'},x})=0$. In words, for $\mu$-almost every $x$, the $\lambda'$-cluster of $x$ is a subset of the $\lambda$-cluster of $x$.
\end{thm}

The discrete case of \cref{thm:chiquet-characterization} (in combination with a condition
for $\lambda$-cohesivity described in \cref{thm:lambda1} below) is
described by \citet{JVZ19} as an ``almost exact characterization''
of the clusters.
Our first main theoretical contribution is an ``exact'' characterization of the minimizer $u_{\mu,\lambda}$. This characterization (\cref{thm:splitcondition} below) seems to
be new even in the discrete case. We need a few definitions and notations. We call a Borel set $V\subset\R^d$ \emph{$\mu$-regular} if either $V$ is a singleton or $\mu(V)>0$. For a $\mu$-regular set $V\subset\mathbf{R}^{d}$,
let
\begin{equation}
	\mathcal{C}_{\mu}(V)\coloneqq\begin{cases}
		\fint_{V}x\,\dif\mu(x) & \text{ if } \mu(V)>0; \\
		x                      & \text{ if } V=\{x\}\end{cases}\label{eq:centroid}
\end{equation}
be the $\mu$-centroid of $V$. (Here and henceforth we write $\fint_{V}f\,\dif\mu\coloneqq \frac{1}{\mu(V)}\int_{V}f\,\dif\mu$.) Note that when $V$ is a singleton with  $\mu(V)>0$ the two cases of \cref{eq:centroid} agree.
\begin{defn}\label{defn:regular}
    We say that a measurable function $u\in L^2(\mu;\R^d)$ is \emph{$\mu$-regular} if there is a measurable representative of $u$ and a Borel set $A\subset\R^d$ such that $\mu(\R^d\setminus A)=0$, $V_{u,x}\cap A$ is $\mu$-regular for $\mu$-a.e.~$x$, and $\mathcal{C}_\mu(V_{u,x}\cap A)\ne\mathcal{C}_\mu(V_{u,z} \cap A)$ for $\mu$-a.e.~$x,z$ with $u(x)\ne u(z)$. If $u$ is $\mu$-regular, we define $\mathcal{E}_{\mu,u}(x)\coloneqq\mathcal{C}_\mu(V_{u,x}\cap A)$, and we note that $\mathcal{E}_{\mu,u}$ is a well-defined element of $L^\infty(\mu;\R^d)$, independent of the choice of $A$ or the choice of representative of $u$. (See \cref{lem:Adoesnotmatter} below.)
	In this case, we let
	\[
		\mathcal{M}_{u}(\mu)\coloneqq(\mathcal{E}_{\mu,u})_*(\mu)=\int\delta_{\mathcal{E}_{\mu,u}(x)}\,\dif\mu(x)
	\]
	be the image of the measure $\mu$ under $\mathcal{E}_{\mu,u}$. By this we mean that for any Borel set $B$, we have
		\[
		\mathcal{M}_{u}(\mu)(B)=\mu(\mathcal{E}_{\mu,u}^{-1}(B)).
	\]
	In words, the measure $\mathcal{M}(u)$ is derived from $\mu$ by concentrating all of the $\mu$-mass
	in each level set of $u$ at the $\mu$-centroid of the level set.
\end{defn}
When the support of $\mu$ is finite, a function $u \colon \supp \mu \to \Rd$ is $\mu$-regular if and only if $\mathcal{C}_\mu(V_{u,x})\ne\mathcal{C}_\mu(V_{u,z} )$ for every $x,z \in \supp \mu$ with $u(x)\ne u(z)$. In words, we ask that different level sets of $u$ have different centroids, and in this case, we have $\mcl M_u(\mu) = \int \delta_{\mcl C_\mu(V_{u,x})} \, \d \mu(x)$. The phrasing of Definition~\ref{defn:regular} is more complicated due to some measure-theoretic technical difficulties that arise when the support of $\mu$ is uncountable. We will prove the following preliminary proposition in \cref{sec:KKT} below.
\begin{prop}\label{prop:regular}
	The function $u_{\mu,\lambda}$ is $\mu$-regular.
\end{prop}
Now we can state our exact characterization of the minimizer $u_{\mu,\lambda}$.
\begin{thm}
	\label{thm:splitcondition}Let $u$ be a $\mu$-regular function and $\lambda\ge0$. The following are equivalent.
	\begin{enumerate}
		\item For $\mu$-a.e.~$x$, we have $V_{u,x} = V_{u_{\mu,\lambda},x}$ up to a $\mu$-null set.
		\item The measure $\mathcal{M}_{u}(\mu)$ is $\lambda$-shattered and,
		      for $\mu$-a.e.~$x$, the restriction $\mu|_{V_{u,x}}$ is $\lambda$-cohesive.
	\end{enumerate}
\end{thm}

Shortly after we posted the first version of this article, \cite{NM21} derived several results on the properties of the optimal clusters. Our framework allows us to recover one of their main results in the measure-valued setting. %
The following proposition, which is analogous to Theorem~3 of \cite{NM21}, states that each cluster is contained in a ball centered at the centroid of the cluster and of radius $\lambda$ times the total mass of the cluster; and that the centroids of the different clusters are sufficiently far apart from one another that these balls do not intersect.
We denote by $\overline B_r(x)$ the closed Euclidean ball of radius $r \ge 0$ centered at $x \in \Rd$.
\begin{prop}\label{prop:ballsdontintersect}
    For $\mu$-a.e.\ $x,z\in\R^d$, we have 
        \begin{equation}
    \label{e.clusterball}
    V_{u_{\mu,\lambda},x} \subset \overline B_{\lambda \mu(V_{u_{\mu,\lambda},x})}\Ll(\mathcal E_{\mu,u_{\mu,\lambda}}(x)\Rr),
    \end{equation}
    and whenever $u_{\mu,\lambda}(x)\ne u_{\mu,\lambda}(z)$,
    \begin{equation}
        |\mathcal{E}_{\mu,u_{\mu,\lambda}}(x)-\mathcal{E}_{\mu,u_{\mu,\lambda}}(z)|> \lambda[\mu(V_{u_{\mu,\lambda},x})+\mu(V_{u_{\mu,\lambda},z})].\label{eq:centroidsfarapart}
    \end{equation}
\end{prop}

We will prove \cref{thm:splitcondition,prop:ballsdontintersect} in \cref{sec:splitcondition}
below.

\cref{thm:splitcondition} motivates taking particular interest in the
properties of $\lambda$-cohesive and $\lambda$-shattered sets. We are mostly interested in situations
in which a dataset can be partitioned into a bounded number of clusters
in the presence of a large number of datapoints. In light of \cref{thm:splitcondition},
this means that there should be a $\lambda$ such that the centroids
of the clusters, weighted by the fraction of datapoints in the cluster,
form a $\lambda$-shattered set, while the datapoints in each
cluster form a $\lambda$-cohesive set. In the regime where there
is a bounded number of clusters but the number of datapoints tends
to infinity, the question of the $\lambda$-shattering of the set of centroids
is a bounded-size optimization problem. In this paper we only address
it in the simplest case. On the other hand, the question of $\lambda$-cohesion
of each cluster lends itself to asymptotic analysis, so this will
interest us in the sequel. We will consider the
``continuum limit'' of situations with continuous measures, and
also provide ``law of large numbers'' results for atomic measures
drawn from the corresponding continuous distributions.

We noted above that if $\mu$ is $\lambda$-cohesive,
then it is also $\lambda'$-cohesive for any $\lambda'\ge\lambda$. By \cref{thm:chiquet-characterization}, this means that if $\mu$ is
$\lambda$-shattered (which \cref{thm:chiquet-characterization,prop:regular} tell us happens if and only if there are no  $\lambda$-cohesive sets of positive $\mu$-measure), then it is also $\lambda'$-shattered for any
$\lambda'\le\lambda$. Thus we define
\begin{equation}\label{eq:lambda1def}
	\lambda_{1}(\mu)\coloneqq\inf\{\lambda\ge0\mid\text{\ensuremath{\mu} is \ensuremath{\lambda}-cohesive}\}
    \end{equation}
and
\begin{equation}\label{eq:lambdastardef}
	\lambda_{*}(\mu)\coloneqq\sup\{\lambda\ge0\mid\text{\ensuremath{\mu} is \ensuremath{\lambda}-shattered}\}.
    \end{equation}
We then say that the level sets of a $\mu$-regular function $u$ are \emph{detectable
	for $\mu$} if
\begin{equation}
	\lambda_{*}(\mathcal{M}_{u}(\mu))>\esssup_{x\sim\mu}\lambda_{1}(\mu|_{V_{u,x}}).
	\label{eq:lambda-detectable}
\end{equation}
By \cref{thm:splitcondition}, this is equivalent to there existing
some $\lambda$ such that the level sets of $u$ are the same (up
to $\mu$-null modifications) as those of $u_{\mu,\lambda}$. We define
the \emph{detection parameter set} to be the  (possibly empty) interval
\begin{equation}
	\Lambda(\mu,u)\coloneqq\left(\esssup_{x\sim\mu}\lambda_{1}(\mu|_{V_{u,x}}),\lambda_{*}(\mathcal{M}_{u}(\mu))\right).\label{eq:Lambdadef}
\end{equation}

The parameter $\lambda_{1}(\mu)$ can be characterized up to a factor of $2$
by simple geometric properties of $\mu$. Define the ``radius'' of the measure $\mu$ by
\begin{equation}
	R(\mu)\coloneqq\esssup_{x\sim\mu}\left|x-\mathcal{C}_{\mu}(\mathbf{R}^{d})\right|,
	\label{eq:Rmudef}
\end{equation} and for $V\subset\mathbf{R}^{d}$, let $\diam V$ denote the Euclidean
diameter of $V$. It turns out (see \cref{prop:lambda1lb} below) that, if $\mu(\R^d)>0$,
\begin{equation}
	\frac{R(\mu)}{\mu(\mathbf{R}^{d})}\le\lambda_{1}(\mu)\le\frac{\diam(\supp\mu)}{\mu(\mathbf{R}^{d})}.\label{eq:lambda1lb}
\end{equation}
Since $R(\mu)\le\diam(\supp\mu)\le2R(\mu)$, this characterizes $\lambda_{1}(\mu)$
up to a factor of $2$  in terms of only the radius and the diameter of
$\supp\mu$. On the other hand, we will compute in \cref{prop:twopoints}
below that, for $a_{0},a_{1}>0$ and $x_{0},x_{1}\in\mathbf{R}^{d}$,
we have
\[
	\lambda_{*}(a_{0}\delta_{x_{0}}+a_{1}\delta_{x_{1}})=\frac{|x_{1}-x_{0}|}{a_{0}+a_{1}}.
\]
Therefore, by \cref{thm:splitcondition}, if equality holds in the first
inequality in \cref{eq:lambda1lb}, then the partition of $\mu+\tau_{x}\mu$---the
sum of $\mu$ and its translation by $x$---into $\supp\mu$ and
$\tau_{x}\supp\mu$ is detectable as long as $|x|>2R(\mu)$. We could
certainly hope for no better since if $|x|\le R(\mu)$ then the supports
of $\mu$ and its translation may overlap (cf. \cref{prop:ballsdontintersect}). On the other hand, if $\lambda_{1}(\mu)>\frac{R(\mu)}{\mu(\mathbf{R}^{d})}$
then for this partition to be detectable we actually need greater
separation than the obvious condition for the supports to not overlap
would suggest. For this reason we are motivated to resolve the value
of $\lambda_{1}(\mu)$ more precisely than is done by \cref{eq:lambda1lb}.
Of particular interest are measures $\mu$ for which $\lambda_{1}(\mu)=\frac{R(\mu)}{\mu(\mathbf{R}^{d})}$,
which are such that combinations with any translation by at least twice the radius
are detectable.

We now state a characterization of $\lambda_{1}(\mu)$, which will follow from a more general theorem (\cref{thm:wcharacterization}
below) giving the KKT characterization of the minimizer of~$J_{\mu,\lambda}$.
(\cref{thm:wcharacterization} will also be crucial for the proof of
\cref{thm:splitcondition}.) In the discrete setting this result follows
from the work of \citet{CGR17}; see also Theorem~1 of \citet{JVZ19}.

\begin{thm}
	\label{thm:lambda1}We have
	\begin{equation}
		\lambda_{1}(\mu)=\mu(\mathbf{R}^{d})^{-1}\min_{q\in\mathcal{Q}(\mu)}\|q\|_{\infty},\label{eq:lambda1cond}
	\end{equation}
	where $\mathcal{Q}(\mu)$ is the set of all $q\in L^{\infty}(\mu^{\otimes2};\Rd)$
	satisfying, for $\mu$-a.e.\ $x,y \in \Rd$,
	\begin{equation}
		q(x,y)=-q(y,x)\label{eq:qantisym}
	\end{equation}
	and
	\begin{equation}
		x-\mathcal{C}_{\mu}(\mathbf{R}^{d})=\fint q(x,z)\,\dif\mu(z).
		\label{eq:qcentroidcond}
	\end{equation}
\end{thm}

We will prove \cref{thm:lambda1} as a consequence of the KKT conditions
in \cref{sec:KKT}.

In \cref{sec:examples}, we use our tools to estimate or compute $\lambda_1(\mu)$ for $\mu$ the uniform measures on the $d$-sphere,
the $d$-ball, and the vertices of the cross-polytope. In $d\ge2$,
these examples do not yield equality in the first inequality of \cref{eq:lambda1lb}.
Thus we also give an explicit example of a nontrivial measure in $d\ge2$
(a ball with density given by a power of the distance from the origin)
for which equality does indeed hold.

In \cref{sec:experiments}, we show the results of some additional numerical experiments regarding the examples considered in \cref{sec:examples}.

\subsection{Stability of the clusters}
We now turn our attention to the stability of the splittings. As the quantities
in \cref{thm:lambda1} are often more analytically tractable in the
presence of symmetries, it can be easier to reason about
the detectability of partitions in the case when measures have a nice symmetry property or a continuous density. On the other hand, in applications one is ultimately interested
in atomic measures, often with some amount of randomness. In \cref{sec:stability}
we prove several stability results showing that the clustering properties
of these models approach the clustering properties of their limits.
As example applications of these results, we prove \cref{t.stoch.ball} as well as the following
theorem.

\begin{thm}
	\label{thm:stability}Let $\mu$ be a probability measure on $\mathbf{R}^{d}$ such that
	\begin{equation}
		\supp\mu=\bigcup_{i=1}^{I}\overline{U_{i}}\label{eq:suppartition}
	\end{equation}
	for some bounded connected open sets $U_{1},\ldots,U_{I}$, each with a Lipschitz
	boundary. Assume that the measure $\mu$ is absolutely continuous with respect to the Lebesgue
	measure, with Radon--Nikodym derivative bounded above and away from
	zero on each $U_{i}$. Let $u$ be an arbitrary function that is constant on each
	$\overline{U_{i}}$, and suppose that $u$ is detectable for $\mu$.
	Let $(X_{n})_{n\ge1}$ be a sequence of independent
	random variables, each with law $\mu$, and define
	\[
		\mu_{N}\coloneqq \frac{1}{N}\sum_{n=1}^{N}\delta_{X_{n}}.
	\]
	Then the endpoints of $\Lambda(\mu_{N},u)$ converge to those of $\Lambda(\mu,u)$
	in probability as $N\to\infty$.
\end{thm}

\cref{thm:stability} is proved in \cref{sec:stability} as a consequence of quantitative continuity estimates for the clustering algorithm with respect to perturbations of $\mu$. Both absolutely continuous and Wasserstein perturbations of $\mu$ are considered; see \cref{prop:stability-AC,prop:Winfty-perturbations,prop:close-shattered}. These propositions can be applied directly to attain stability results analogous to \cref{thm:stability} for other random configurations, or to obtain quantitative results for finite numbers of datapoints.

Several variants of the clustering method discussed in this paper can also be considered. For instance, in the fusion term $\iint |u(x)-u(y)|\,\dif\mu(x)\,\dif\mu(y)$ appearing in \eqref{e.def.J}, one can consider replacing the Euclidean norm $|\cdot|$ by another norm, such as the $\ell^1$ norm. While this modification may be interesting from a computational perspective, it will also destroy the rotational invariance of the functional $J_{\mu,\lambda}$, and in general, we expect that these modified methods will also fail to correctly resolve the stochastic ball model with nearby balls. Another possibility is to introduce weights in the fusion term, such as
\begin{equation*}  %
	\iint_{x \neq y}|x-y|^{-\alpha} |u(x) - u(y)| \, \d \mu(x) \, \d \mu(y),
\end{equation*}
for some exponent $\alpha \in (0,d)$ to be decided. The choice of a power-law weight can be motivated by the desire to ensure that the set of partitions discovered by the algorithm as we vary $\lambda$ is only rescaled under a rescaling of the measure; if one has in mind possibly  complex datasets involving multiple scales, this seems like a natural requirement. Alternative possibilities that do not satisfy this property include replacing $|x-y|^{-\alpha}$ by $\exp(-c |x-y|)$, or other decreasing functions of the distance $|x-y|$. In the discrete setting, one can enforce stronger locality by restricting the sum to connected pairs in the $k$-nearest-neighbor graph. The latter possibility offers significant computational benefits, see \cite{chi2015splitting}. After posting the first ArXiv version of this paper, we showed in  \cite{local-clustering} that the introduction of suitably adjusted exponential weights allows us to recover very general cluster shapes. In particular, the SON clustering algorithm with suitably adjusted weights succeeds in identifying disjoint balls in stochastic ball models, no matter how close they are; and it can also recover clusters whose convex hulls interesect. This contrasts with the results stated in \cref{t.stoch.ball} and \cref{prop:ballsdontintersect} for the unweighted SON clustering algorithm. On the other hand, the addition of weights breaks the symmetries that allow us to prove the theoretical results in the present work.

\section{Examples}\label{sec:examples}

In this section we compute $\lambda_{1}(\mu)$ for several choices
of $\mu$.
\begin{prop}[Two points]
	\label{prop:twopoints}Let $x_{0},x_{1}\in\mathbf{R}^{d}$, $a_0,a_1>0$, and let
	$\mu=a_{0}\delta_{x_{0}}+a_{1}\delta_{x_{1}}$. Then
	\begin{equation}
		\lambda_{1}(\mu)=\lambda_{*}(\mu)=\frac{|x_{1}-x_{0}|}{a_{0}+a_{1}}.\label{eq:twopts}
	\end{equation}
\end{prop}

\begin{proof}
	Since the support of $\mu$ has only two points, it is clear that
	$\lambda_{1}(\mu)=\lambda_{*}(\mu)$. (For a given~$\lambda$, either
	$\mu$ is $\lambda$-cohesive or it is $\lambda$-shattered.) We observe that 
	\begin{equation*}  %
	\mathcal{C}_{\mu}(\mathbf{R}^{d}) = \frac{a_{0}x_{0}+a_{1}x_{1}}{a_{0}+a_{1}}.
	\end{equation*}
	For a function $q$ to satisfy \cref{eq:qantisym}--\cref{eq:qcentroidcond}, we must have that
	\[
		x_{0}-\frac{a_{0}x_{0}+a_{1}x_{1}}{a_{0}+a_{1}}=\frac{a_{1}}{a_{0}+a_{1}}[x_{0}-x_{1}]=\fint q(x_{0},y)\,\dif\mu(y) = \frac{a_1}{a_0 + a_1} q(x_0,x_1)
	\]
	and
	\[
		x_{1}-\frac{a_{0}x_{0}+a_{1}x_{1}}{a_{0}+a_{1}}=\frac{a_{0}}{a_{0}+a_{1}}[x_{1}-x_{0}]=\fint q(x_{1},y)\,\dif\mu(y) = \frac{a_0}{a_0 + a_1} q(x_1, x_0).
	\]
	The only function $q$ that satisfies the conditions \cref{eq:qantisym}--\cref{eq:qcentroidcond} is therefore the function $q(x,y) \coloneqq x-y$. Then \cref{eq:twopts} follows from \cref{thm:lambda1}.
\end{proof}
\begin{prop}[Interval]
	\label{prop:balld1}Let $d=1$ and let $\mu$ be the Lebesgue measure
	on $[-1/2,1/2]$ (with total mass $1$). Then $\lambda_{1}(\mu)=1/2$.
\end{prop}

\begin{proof}
	Note that $\mathcal{C}_{\mu}(\mathbf{R}^{d})=0$. Letting $q(x,y)\coloneqq\frac{1}{2}\sgn(x-y)$, we have
	\[
		\int_{-\frac{1}{2}}^{\frac{1}{2}}\frac{1}{2}\sgn(x-y)\,\dif y=\frac{1}{2}\left[(x-(-1/2))-(1/2-x)\right]=x,
	\]
	so \cref{eq:qcentroidcond} holds, and $\|q\|_{\infty}=1/2$, which
	means that $\lambda_{1}\le1/2$ by \cref{thm:lambda1}. On the other
	hand, \cref{eq:lambda1lb} shows that $\lambda_{1}(\mu)\ge1/2$, so
	in fact $\lambda_{1}(\mu)=1/2$.
\end{proof}
The next proposition is a characterization of $\lambda_1(\mu)$ for measures $\mu$ with support in the unit sphere that satisfy certain symmetry properties. We will next apply this result to several concrete examples.
\begin{prop}[Symmetric measures]
	\label{prop:lambda1spheresymmetric}Suppose that $\mu$ is supported
	on $S^{d-1}=\partial B_{1}(0)\subset\mathbf{R}^{d}$, the support
	of $\mu$ comprises at least two points, and there is a subgroup $G\subset\mathrm{O}(d)$
	(the group of Euclidean isometries of $\mathbf{R}^{d}$ preserving
	the origin) preserving $\mu$, %
	acting transitively on $\supp\mu$, and such that for each $x\in\supp\mu$ and each $y\in S^{d-1}\setminus \{x,-x\}$, there is a $g\in G$ such that $g\cdot x=x$ but $g\cdot y\ne y$. Then for every $y\in\supp\mu$
	we have
	\begin{equation}
		\lambda_{1}(\mu)=\frac{2}{\int|x-y|\,\dif\mu(x)}\label{eq:lambda1symmetric}
	\end{equation}
	and
	\begin{equation}
		\lambda_{1}(\mu)\mu(\mathbf{R}^{d})\ge\sqrt{2}.\label{eq:NOTOUCHING}
	\end{equation}
\end{prop}

\begin{proof}
	The strict convexity of $J_{\mu,\lambda}$ noted in the introduction
	implies that the minimizer $u_{\mu,\lambda}$ is unique. Since the measure $\mu$ is invariant under the action of $G$, the minimizer $u_{\mu,\lambda}$ must also be invariant under the action of $G$, in the sense that, for every $g \in G$ and $\mu$-a.e.\ $x \in \Rd$, we have
	\begin{equation}
	\label{e.u.g.invariance}
	    u_{\mu,\lambda}(g\cdot x) = g \cdot u_{\mu,\lambda}(x).
	\end{equation}
	For each $x\in \supp \mu$, if $u_{\mu,\lambda}(x)\not\in\R x$, then by assumption there is a $g\in G$ such that $g\cdot x = x$ and $g\cdot u_{\mu,\lambda}(x)\ne u_{\mu,\lambda}(x)$; but this would imply that $u_{\mu,\lambda}(x) = u_{\mu,\lambda}(g\cdot x) = g\cdot u_{\mu,\lambda}(x)\ne u_{\mu,\lambda}(x)$, a contradiction. Therefore, $u_{\mu,\lambda}(x)\in \R x$ for $\mu$-a.e.\ 
$x \in \Rd$. In other words, for $\mu$-a.e.\ $x \in \Rd$, we can find some $a_{\lambda,x} \in \R$ such that $u_{\mu,\lambda}(x) = a_{\lambda,x} x$. Using again \cref{e.u.g.invariance}, we deduce that for every $g \in G$, we must have $u_{\mu,\lambda}(g\cdot x) = g\cdot u_{\mu,\lambda}(x) = a_{\lambda,x} g \cdot x$. By the transitivity of the action of $G$ on $\supp\mu$, we must thus therefore have a fixed $a_\lambda\in\R$, depending only on $\lambda$ and not on $x$, such that 
	$u_{\mu,\lambda}(x)=a_{\lambda}x$ for $\mu$-a.e.\ $x \in \Rd$. %
	Since $\mu$ is invariant under the action of $G$, which acts transitively on $\supp \mu$, we have that the integral
	$\int |x-y| \, \d \mu(x)$
	does not depend on the choice of $y \in \supp \mu$. Recalling also that $\supp \mu \subset S^{d-1}$, we see that, for every $a \in \R$ and an arbitrary $y\in\supp\mu$,
	\begin{align}
		J_{\mu,\lambda}(x\mapsto ax) 
		& = \int |ax-x|^2 \, \d \mu(x) + \lambda \iint |ax-az| \, \d \mu(x) \, \d \mu(z) \\
		&  =\mu(\mathbf{R}^{d})\left[a^{2}+\lambda|a|\int|x-y|\,\dif\mu(x)-2a+1\right].
		\label{eq:Gforinvariant}
	\end{align}
	The function $u_{\mu,\lambda}$ is constant if and only if the quantity in \eqref{eq:Gforinvariant} is minimized for $a = 0$. This occurs 
	exactly when
	\[
		\lambda\ge\frac{2}{\int|x-y|\,\dif\mu(x)},
	\]
	and we have therefore shown \cref{eq:lambda1symmetric}.

	We now argue that $\mathcal{C}_{\mu}(\mathbf{R}^{d}) = 0$. Integrating the identity \eqref{e.u.g.invariance} in $x$, we see that $\mathcal{C}_{\mu}(\mathbf{R}^{d})$ must be a fixed point of the action of the group $G$. If $\supp \mu$ is of the form $\{x,-x\}$ for some $x \in \Rd$, then by transitivity the measure $\mu$ places the same mass on $x$ and $-x$, so $\mathcal{C}_{\mu}(\mathbf{R}^{d}) = 0$. Otherwise, we observe that the group $G$ has no other fixed point than the origin.  Indeed, if $G$ had another fixed point, then by scaling we could obtain a fixed point $y \in S^{d-1}$. Since $\supp \mu$ is not of the form $\{y,-y\}$, we can find some $x \in \supp \mu \setminus \{y,-y\}$. The assumption on $G$ then guarantees the existence of some $g \in G$ with $g\cdot y \neq y$, a contradiction.

	Now that $\mathcal{C}_{\mu}(\mathbf{R}^{d})=0$ is established, we apply Jensen's inequality to get that 
	\begin{align*}
		\frac{1}{\mu(\mathbf{R}^{d})}\int|x-y|\,\dif\mu(x) & \le\left(\frac{1}{\mu(\mathbf{R}^{d})}\int|x-y|^{2}\,\dif\mu(x)\right)^{1/2}             \\
		                                                   & =\left(\frac{1}{\mu(\mathbf{R}^{d})}\int2(1-x\cdot y)\,\dif\mu(x)\right)^{1/2}
		                                                   \\
		                                                   & 
= \left(2 - \mathcal{C}_{\mu}(\mathbf{R}^{d}) \cdot y \right)^{1/2}                       =\sqrt{2}.
	\end{align*}
	Combining this with \cref{eq:lambda1symmetric} yields that
	\[
		\lambda_{1}(\mu)\mu(\mathbf{R}^{d})\ge\frac{2\mu(\mathbf{R}^{d})}{\int|x-y|\,\dif\mu(x)}\ge\sqrt{2}.\qedhere
	\]
\end{proof}
\begin{cor}[$d$-sphere]
	\label{cor:dsphere}Suppose that $d\ge2$ and let $\mu$ be the uniform
	measure on the unit sphere $S^{d-1}=\partial B_{1}(0)$. Then
	\begin{equation}
		\lambda_{1}(\mu)\mu(\mathbf{R}^{d})=\frac{\Gamma(d-1/2)\Gamma((d-1)/2)}{\Gamma(d-1)\Gamma(d/2)},\label{eq:lambda1sphere}
	\end{equation}
	where $\Gamma(z)=\int_{0}^{\infty}t^{z-1}\e^{-t}\,\dif t$ denotes
	the standard gamma function. In particular,
	\begin{equation}
		\lim_{d\to\infty}\lambda_{1}(\mu)\mu(\mathbf{R}^{d})=\sqrt{2}.\label{eq:lambda1spherelimit}
	\end{equation}
\end{cor}

\begin{proof}
	Assume without loss of generality that $\mu(\Rd)$ is the area of $S^{d-1}$, that is,
	\[
		\mu(\mathbf{R}^{d})=\frac{2\pi^{d/2}}{\Gamma(d/2)}.
	\]
	We also have
	\begin{align*}
		\int|\mathbf{e}_{1}-x|\,\dif\mu(x) & =\frac{2\pi^{(d-1)/2}}{\Gamma((d-1)/2)}\int_{0}^{\pi}(1-\cos^{2}\theta)^{\frac{d-2}{2}}\sqrt{(\cos\theta-1)^{2}+\sin^{2}\theta}\,\dif\theta \\
		                                   & =\frac{2^{d}\pi^{(d-1)/2}}{\Gamma((d-1)/2)}\int_{0}^{\pi}\sin^{d-1}(\theta/2)\cos^{d-2}(\theta/2)\,\dif\theta                               \\
		                                   & =\frac{2^{d}\pi^{(d-1)/2}}{\Gamma((d-1)/2)}\int_{0}^{1}t^{d/2-1}(1-t)^{(d-3)/2}\,\dif t                                                     \\
		                                   & =\frac{2^{d}\pi^{(d-1)/2}\Gamma(d/2)}{\Gamma(d-1/2)}                                                                                        \\
		                                   & =\frac{4\pi^{d/2}\Gamma(d-1)}{\Gamma((d-1)/2)\Gamma(d-1/2)}.
	\end{align*}
	The second identity is by the half-angle formulas for sine and cosine, the third is by making the substitution $t=\sin^2(\theta/2)$, the fourth is by the standard formula for the beta integral, and the last is by the Legendre duplication formula. Hence \cref{eq:lambda1sphere} follows from \cref{prop:lambda1spheresymmetric}, noting that the group $G$ can be taken to be all of $\mathrm{O}(d)$, which clearly satisfies the hypotheses.
	The limit~\cref{eq:lambda1spherelimit} is then a simple computation using Stirling's approximation.
\end{proof}
\begin{cor}[Vertices of the $n$-gon]\label{cor:n-gon} Let $d=2$, $n\ge 2$, and let $\mu$ be a uniform measure on the vertices of the regular $n$-gon inscribed in the unit circle, namely
	\[\mu = \frac1n\sum_{j=1}^n \delta_{\e^{2\pi\ii j/n}},\]
	where we identify $\R^2$ with $\C$. Then we have\[\lambda_1(\mu)\mu(\Rd) = n\tan\left(\frac\pi{2n}\right).\]
\end{cor}
\begin{proof}
	We have 
	\[\frac1{\mu(\Rd))}\int|x-y|\,\dif \mu(x) = \frac1n\sum_{j=1}^n |1-\e^{2\pi\ii j/n}|=\frac2n\sum_{j=1}^n \sin(\pi j/n) =\frac2n\cot\left(\frac\pi{2n}\right),\] and the result follows from \cref{prop:lambda1spheresymmetric}.
\end{proof}
\begin{cor}[Vertices of the cross-polytope]
	\label{cor:crosspolytope}Consider the measure on $\mathbf{R}^{d}$
	given by
	\[
		\mu=\sum_{i=1}^{d}[\delta_{\mathbf{e}_{i}}+\delta_{-\mathbf{e}_{i}}].
	\]
	Then
	\[
		\lambda_{1}(\mu)\mu(\mathbf{R}^{d})=\frac{2d}{(d-1)\sqrt{2}+1}
	\]
	and in particular
	\[
		\lim_{d\to\infty}\lambda_{1}(\mu)\mu(\mathbf{R}^{d})=\sqrt{2}.
	\]
\end{cor}

\begin{proof}
	We have
	\[
		\int|\mathbf{e}_{1}-x|\,\dif\mu(x)=2(d-1)\sqrt{2}+2
	\]
	and the result follows from \cref{prop:lambda1spheresymmetric}.
\end{proof}
\begin{prop}[$d$-ball]
	\label{p.ball}
	Let $\gamma_d$ be as defined in \eqref{e.def.gamma}, and $\mu$ be a uniform measure on the unit ball $B_1(0) \subset \Rd$. Then
	\begin{equation}
		\label{e.ball}
		\gamma_d \le \lambda_1(\mu)  \mu(\Rd) \le 2^{1-\frac 1 d}.
	\end{equation}
\end{prop}
\begin{proof}
	Similarly to the proof of \cref{prop:lambda1spheresymmetric}, we start by computing, for every $a \ge 0$,
	\begin{equation*}  %
		J_{\mu,\lambda}(x\mapsto ax)  = (1-a)^2 \int |x|^2 \, \d \mu(x)+ \lambda a \iint |x-y| \, \d \mu(x) \, \d \mu(y).
	\end{equation*}
	If the ball is $\lambda$-cohesive, then the quantity above must be minimal when $a = 0$. In such a case, we must have
	\begin{equation*}  %
		\lambda \ge \frac{2 \int |x|^2 \, \d \mu(x)}{\iint |x-y| \, \d \mu(x) \, \d \mu(y)}.
	\end{equation*}
	In other words, we have
	\begin{equation}
		\label{e.lowerbound.lambda1}
		\lambda_1(\mu) \ge \frac{2 \int |x|^2 \, \d \mu(x)}{\iint |x-y| \, \d \mu(x) \, \d \mu(y)}.
	\end{equation}
	The numerator in \eqref{e.lowerbound.lambda1} is
	\begin{equation}
		\label{e.lb.numerator}
		\mu(\Rd) \frac{\int_0^1 r^{2+d-1} \, \d r}{\int_0^1 r^{d-1} \, \d r} = \mu(\Rd) \frac{d}{d+2}.
	\end{equation}
	Denoting
	\begin{equation*}  %
		\beta_d \coloneqq \fint \fint |x-y| \, \d \mu(x) \, \d \mu(y),
	\end{equation*}
	we have that
	\begin{equation*}
		\beta_d = \frac{2d}{2d+1} \cdot
		\begin{cases}
			\frac{2^{3d+1}((d/2)!)^2 d!}{(d+1)(2d)!\pi}      & \mbox{ if  $d$ is even}, \\
			\frac{2^{d+1}(d!)^3}{(d+1)(( (d-1)/2 )!)^2(2d)!} & \mbox{ if  $d$ is odd}.
		\end{cases}
	\end{equation*}
	For $d = 2$, the proof of this identity can be found in \cite{Du97}, \citet[Exercise~4.13.4]{GSbook}, or \citet[Section~4.2]{Sabook}. In higher dimension, the computation  is only sketched in \cite{Du97}, but does not pose additional difficulties (the high-dimensional integral splits into a product of Wallis integrals). One can verify that, for every $d \ge 1$,
	\begin{equation*}  %
		\frac{\beta_{d+2}}{\beta_d} = \frac{(2d+2)(2d+4)^3}{2d(2d+3)(2d+5)(2d+6)} = 1 + \frac{9d^2 + 35d + 32}{d(2d+3)(2d+5)(d+3)}.
	\end{equation*}
	Combining this with \eqref{e.lowerbound.lambda1} and \eqref{e.lb.numerator}, we obtain the first inequality in \eqref{e.ball}.

	For the second inequality in \eqref{e.ball}, if $d=1$ then the inequality follows from \cref{prop:balld1}, so assume that $d\ge 2$.
	Fix $\alpha\in\mathbf{R}$ to be chosen later and set
	\[
		q_{1}(x,y)=\begin{cases}
			\alpha\sgn(x)  & \text{ if } |x|>|y|; \\
			-\alpha\sgn(y) & \text{ if } |x|<|y|; \\
			0              & \text{ if } |x|=|y|.
		\end{cases}
	\]
	Then we have
	\[
		\fint q_{1}(x,y)\,\dif\mu(y)=\alpha\frac{\mu\{y:|y|<|x|\}}{\mu(B_{1}(0))}\sgn(x)=\alpha|x|^{d}\sgn(x)=\alpha|x|^{d-1}x.
	\]
	Let $x,y\in B_{1}(0)$ with $|x|>|y|$. We have
	\begin{align*}
		 & \left|q_{1}(x,y)+x-y-\fint q_{1}(x,z)\,\dif\mu(z)+\fint q_{1}(y,z)\,\dif\mu(z)\right|                        \\
		 & \qquad=\left|\alpha\sgn(x)+x-y-\alpha|x|^{d-1}x+\alpha |y|^{d-1}y\right|                                     \\
		 & \qquad=\left|\sgn(x)[\alpha+|x|-\alpha|x|^{d}]-\sgn(y)[|y|-\alpha|y|^{d}]\right|                             \\
		 & \qquad\le\alpha + \left||x|-\alpha|x|^{d}\right|+\left||y|-\alpha|y|^{d}\right|                              \\
		 & \qquad\le\alpha+2\left(\frac{1}{(\alpha d)^{\frac{1}{d-1}}}-\frac{\alpha}{(\alpha d)^{\frac{d}{d-1}}}\right) \\
		 & \qquad=\alpha+\frac{2}{(\alpha d)^{\frac{1}{d-1}}}\left(1-\frac{1}{d}\right).
	\end{align*}
	Now taking $\alpha = 2^{\frac{d-1}{d}}/d$,
	we get
	\[
		\left|q_{1}(x,y)+x-y-\fint q_{1}(x,z)\,\dif\mu(z)+\fint q_{1}(y,z)\,\dif\mu(z)\right|\le\frac{2^{\frac{d-1}{d}}}{d}+2^{1-1/d}\left(1-\frac{1}{d}\right)=2^{1-1/d}.
	\]
	Thus by \cref{prop:patchupthedifference} we have
	\[
		\lambda_{1}(B_{1}(0))\le\mu(B_{1}(0))^{-1}2^{1-1/d}.\qedhere
	\]
\end{proof}

\begin{prop}[Power-law weighted ball]
	\label{prop:reweightedball}Let $R\in(0,\infty)$ and $\mu$ be the measure given by
	\[
		\dif\mu(x)=|x|^{-(d-1)}\mathbf{1}\{|x|\le R\}\dif x.
	\]
	Then
	\[
		\lambda_{1}(\mu)=\frac{R(\mu)}{\mu(\mathbf{R}^{d})}=\frac{2}{\alpha_{d-1}},
	\]
	where $\alpha_{d-1}=\frac{2\pi^{d/2}}{\Gamma(d/2)}$ is the area of
	the unit $(d-1)$-sphere.
\end{prop}

\begin{proof}
	We first note that, for any $s\in[0,R]$, we have using spherical
	coordinates that
	\begin{align*}
		\mu(B_{s}(0)) & =\int_{0}^{s}\int_{S^{d-1}}\dif\mathcal{H}^{d-1}(\boldsymbol{\theta})\,\dif r=\frac{1}{2}\alpha_{d-1}s,
	\end{align*}
	Define
	\[
		q(x,y)=\begin{cases}
			R\sgn(x)  & \text{ if } |x|>|y|; \\
			-R\sgn(y) & \text{ if } |x|<|y|; \\
			0         & |x|=1.
		\end{cases}
	\]
	Then $q$ is evidently antisymmetric and $\|q\|_{\infty}=R$, and
	we have, using spherical coordinates and symmetry, that
	\begin{align*}
		\fint_{\mathbf{R}^{2}}q(x,y)\,\dif\mu(y) & =\frac{1}{\mu(B_{R}(0))}\int_{0}^{R}\int_{S^{d-1}}q(x,r\boldsymbol{\theta})\,\dif\mathcal{H}^{d-1}(\boldsymbol{\theta})\,\dif r=R\sgn(x)\frac{\mu(B_{|x|}(0))}{\mu(B_{R}(0))}=x.
	\end{align*}
	By \cref{thm:lambda1} this implies that
	\[
		\lambda_{1}\le\frac{R}{\frac{1}{2}\alpha_{d-1}R}=\frac{2}{\alpha_{d-1}}.
	\]
	On the other hand, we have by \cref{prop:lambda1lb} that
	\[
		\lambda_{1}\ge\frac{R}{\mu(\mathbf{R}^{d})}=\frac{2}{\alpha_{d-1}}.\qedhere
	\]
\end{proof}

\section{Numerical experiments}\label{sec:experiments}
In this section we supplement our theoretical results with some numerical experiments; see also \cref{fig:lambda-effect-explanation,fig:stochastic-ball-model-simple}. The code is available at 
\begin{center}
\url{https://github.com/ajdunlap/son-clustering-experiments}. 
\end{center}
Our experiments were performed using the algorithm of \cite{JV20}. This algorithm provides a certificate that the ouput clustering is correct. When~$\lambda$ is very close to a value at which the number of clusters changes, limitations on computer time and numerical accuracy make it difficult to perform the calculations to sufficient accuracy to obtain the certificate. In particular, for situations such as that described by Theorem~\ref{t.stoch.ball}, the SON clustering algorithm becomes numerically very challenging to resolve for $\lambda$ close to $\lambda_c$, while the clustering structures that are produced for other values of $\lambda$ are not the expected partition into two parts. This further clarifies how the SON algorithm fails to resolve this clustering problem successfully in practice. 
 Further work would be required to numerically probe the behavior of the algorithm very close to these critical values of $\lambda$.

\subsection{Polygons} We begin with a case in which we can theoretically compute everything exactly. Fix some integer $n$ and let $\mu$ be a probability measure given by a Dirac mass at each vertex of two regular $n$-gons (each inscribed in a circle of radius $1$) whose centers are separated by a distance $2r$. Our clustering characterization \cref{thm:splitcondition}, combined with \cref{prop:twopoints,cor:n-gon}, tell us that sum-of-norms clustering makes exactly one cluster from each $n$-gon exactly when $2n\tan\left(\frac\pi{2n}\right) < \lambda < 2r$. We test this numerically with $n=8$ and $r=1.7$.  In this case, $2n\tan\left(\frac\pi{2n}\right) \simeq 3.18$. We perform simulations with $\lambda=3.1,3.3,3.5$ (noting that $3.1< 2n\tan\left(\frac\pi{2n}\right)<3.3<2r<3.5$) and show the results in \cref{fig:two-octagons}. We see that our theoretical results are matched by the experiments.

\begin{figure}
    \centering
    \begin{subfigure}[t]{2.9in}
        \centering
        \input{plots/two-octagons-3.10.tex}
        \caption{$\lambda = 3.1$. Each point is in its own cluster.}
    \end{subfigure}
    \hfill
    \begin{subfigure}[t]{2.9in}
        \centering
        \input{plots/two-octagons-3.30.tex}
        \caption{$\lambda = 3.3$. Two clusters.}
    \end{subfigure}
    \hfill
    \begin{subfigure}[t]{2.9in}
        \centering
        \input{plots/two-octagons-3.50.tex}
        \caption{$\lambda = 3.5$. One cluster.}
    \end{subfigure}
    \hfill
    \caption{\label{fig:two-octagons}Clustering results for the vertices of two octagons. Vertices assigned to the same cluster are drawn in the same color.}
\end{figure}

\subsection{\texorpdfstring{$\lambda_1$}{λ₁} for a ball}
\cref{p.ball} does not precisely determine $\lambda_1(\mu)$ where $\mu$ is the indicator function of the unit ball. Here we perform a numerical experiment to estimate $\lambda_1(\mu)$ in dimension $d=2$. We approximate the interior of the ball by the set of all points on a rectangular lattice with spacing $\delta$ lying inside the ball, i.e. $\{x\in\delta\Z^2\mid |x|\le 1\}$, and compute the number of clusters for varying choices of $\lambda$. The results are shown in \cref{fig:find-lambda1}. In view of \cref{cor:dsphere,eq:gammas,prop:Winfty-perturbations} below, we know that the limit as $\delta\searrow 0$ of $\lambda_1$ is between $1.104\ldots$ and $1.414\ldots$. The results of \cref{fig:find-lambda1} are roughly consistent with this, and suggest that the true limit is closer to the lower end of the theoretically proved range. The numerical results also suggest that $\lambda_1$ and $\lambda_*$ may be equal for the ball, which has not been studied theoretically, and thus is an interesting conjecture.

In \cref{fig:find-lambda1}, the scheme of \cite{JV20} is again used to compute the clusterings. When $\lambda$ is close to a value at which the number of clusters changes, the certification procedure of \cite{JV20} may fail even when the duality gap in the clustering algorithm is close to machine precision. This is the reason for the missing values in the figure. Using a more sophisticated algorithm to more precisely estimate the values of $\lambda_1$ and $\lambda_*$ for the ball is an interesting topic for future work.

\begin{figure}
    \vspace{-26pt}
    \foreach\mydelta in {0.1150,0.1075,0.1000,0.0925,0.0850,0.0775,0.0700,0.0625}{
        \begin{subfigure}[t]{2.9in}
            \centering
            \vskip 0.33in
            \input{plots/find-lambda1-\mydelta.tex}
        \caption{$\delta=\mydelta$}
    \end{subfigure}
    \hfill
    }
    \caption{\label{fig:find-lambda1}The number of clusters produced by sum-of-norms clustering run on the measure $\mu$ given by the uniform distribution on $\{x\in\delta\Z^2\mid |x|\le 1\}$, for varying choices of $\lambda$ and $\delta$. Missing values correspond to failures to certify the clustering using the procedure of \cite{JV20}.}
\end{figure}

\section{KKT characterization of the minimizer\label{sec:KKT}}

Recall that, for convenience, we assume throughout the paper that the measure $\mu$ is finite (meaning that $\mu(\Rd) < \infty$) and has compact support.
We start by justifying the existence and uniqueness of a minimizer for $J_{\mu,\lambda}$. It is clear (or see \cref{lem:Jcts} below) that the functional $J_{\mu,\lambda}$ is continuous on $L^{2}(\mu;\R^d)$. Moreover, $J_{\mu,\lambda}$ is uniformly convex: for every $u,v\in L^{2}(\mu;\R^d)$, we have
\begin{equation}
	\frac{1}{2}\left(J_{\mu,\lambda}(u+v)+J_{\mu,\lambda}(u-v)\right)-J_{\mu,\lambda}(u)\ge\int v^{2}\,\dif\mu.\label{eq:convexity}
\end{equation}
Finally, it is clear that the functional $J_{\mu,\lambda}$ is coercive,
i.e.\ that there exist  $c_1>0$ and $c_2\ge 0$ such that $J_{\mu,\lambda}(u)\ge  c_1 \|u\|_{L^2(\mu;\R^d)}^2-c_2$ for all $u\in L^2(\mu;\R^d)$.
Thus there exists a unique minimizer $u_{\mu,\lambda}\in L^{2}(\mu;\R^d)$
for $J_{\mu,\lambda}$ \cite[Section~8.2]{evansbook}.

The key to most of our analysis is the following theorem, which evaluates
the subdifferential of $J_{\mu,\lambda}$ and derives the resulting
KKT characterization of the minimizer. For each $z \in \Rd \setminus \{0\}$, we write
\begin{equation}
	\label{e.def.sgn}
	\sgn(z) \coloneqq \frac z {|z|}.
\end{equation}

\begin{thm}
	\label{thm:wcharacterization}Let $u\in L^{2}(\mu;\Rd)$. We have $u=u_{\mu,\lambda}$ ($\mu$-a.e.) if and only if
	there exists $w\in L^{\infty}(\mu^{\otimes2};\Rd)$ such that, for $\mu$-a.e.~$x,y\in\mathbf{R}^{d}$,
	we have
	\begin{equation}
		w(x,y)=-w(y,x),\label{eq:wantisymmetric}
	\end{equation}
	\begin{equation}
		u(x)\ne u(y)\implies w(x,y)=\sgn(u(x)-u(y)),\label{eq:sgnfordifferent}
	\end{equation}
	\begin{equation}
		|w(x,y)|\le1,\label{eq:wsup}
	\end{equation}
	and
	\begin{equation}
		x-u(x)=\lambda\int w(x,z)\,\dif\mu(z).\label{eq:xminmusux}
	\end{equation}
\end{thm}

\begin{proof}
	For every measure $\nu$ and functional $F\colon L^{2}(\nu;\Rd)\to\mathbf{R}$,
	we define the subdifferential \cite[Section~I.5]{ETbook} of $F$ at $u\in L^{2}(\nu;\Rd)$ by
	\begin{equation}
		\partial F(u)\coloneqq\left\{ p\in L^{2}(\nu;\Rd)\ :\ \forall v\in L^{2}(\nu;\Rd),\ F(u+v)\ge F(u)+\int p\cdot v\,\dif\nu\right\} .\label{eq:subdifferential}
	\end{equation}

	\emph{Step 1}. In this step, for every probability measure $\nu$
	on $\mathbf{R}^{d}$ with compact support, we identify the subdifferential
	of the functional
	\begin{equation}
		F(u)\coloneqq\int|u|\,\dif\nu\label{eq:Fdef}
	\end{equation}
	at $u\in L^{2}(\nu;\Rd)$ as
        \begin{equation}
		\label{eq:subdifferential-claim}
                \begin{aligned}
                &\partial F(u) =
		\\
                &\left\{ w\in L^{\infty}(\nu;\Rd)\ :\ \|w\|_{L^{\infty}}\le1\text{ and for \ensuremath{\nu}-a.e.~\ensuremath{x\in\mathbf{R}^{d}}, }u(x)\ne0\implies w(x)=\sgn(u(x))\right\} .
                \end{aligned}
            \end{equation}
	We denote by $K_{1}(u)$ the set on the right side of \cref{eq:subdifferential-claim}.
	Note that for every $a,b,w\in\mathbf{R}^{d}$, if $|w|\le1$
	satisfies
	\[
		a\ne0\implies w=\sgn(a),
	\]
	then
	\[
		|a+b|\ge|a|+w\cdot b.
	\]
	From this observation, we can verify that $K_{1}(u)\subseteq\partial F(u)$
	directly from \cref{eq:subdifferential} and \cref{eq:subdifferential-claim}.
	In order to show the opposite inclusion, we argue by contradiction
	and suppose that there exists $p\in\partial F(u)\setminus K_{1}(u)$.
	Since $K_{1}(u)$ is convex and closed in the Hilbert space $L^{2}(\nu; \Rd)$,
	the hyperplane separation theorem \cite[Section~I.1]{ETbook} guarantees the existence of a function
	$v\in L^{2}(\nu; \Rd)$ such that
	\begin{equation}
		\int p\cdot v\,\dif\nu>\sup_{w\in K_{1}(u)}\int w\cdot v\,\dif\nu.\label{eq:hyperplane-separation}
	\end{equation}
	Defining $w\in L^{\infty}(\nu; \Rd)$ by
	\[
		w(x)=\begin{cases}
			\sgn(u(x)) & \text{ if }u(x)\ne0; \\
			\sgn(v(x)) & \text{ otherwise},
		\end{cases}
	\]
	we have for every $\eps>0$ that
	\[
		\eps^{-1}(F(u+\eps v)-F(u))=\int w\cdot v\,\dif\nu+\int r_{\eps}\,\dif\nu,
	\]
	where
	\[
		r_{\eps}=\eps^{-1}(|u+\eps v|-|u|-\eps w\cdot v).
	\]
        At a point where $u=0$, we have $r_\eps = |v|-\sgn(v)\cdot v = 0$ by the definitions, while at a point where $u\ne 0$, we have $r_\eps = 0$ for $\eps$ sufficiently small by the local linearity of $|\cdot|$, so
	the function $r_{\eps}$ tends to $0$ $\nu$-a.e.\ as $\eps\downarrow0$. Moreover,
	by the Cauchy--Schwarz and triangle inequalities we see that
        $|r_{\eps}|\le \eps^{-1}(|u|+\eps|v|-|u|+\eps|w||v|)=2|v|$. It thus follows from dominated convergence that
	\[
		\lim_{\eps\downarrow0}\eps^{-1}(F(u+\eps v)-F(v))=\int w\cdot v\,\dif\nu.
	\]
	On the other hand, recalling that $p\in\partial F(u)$, we must also
	have for every $\eps>0$ that
	\[
		\eps^{-1}(F(u+\eps v)-F(u))\ge\int p\cdot v\,\dif\nu.
	\]
	But the two previous displays contradict \cref{eq:hyperplane-separation}.

	\emph{Step 2}. In this step, we show that the subdifferential of the functional
	\[
		G(u)\coloneqq\int|u(x)-u(y)|\,\dif\mu(x)\,\dif\mu(y)
	\]
	at $u\in L^{2}(\mu; \Rd)$ is given by
	\begin{equation}
		\partial G(u)=\left\{ x\mapsto2\int w(x,y)\,\dif\mu(y)\ :\ w\text{ satisfies \cref{eq:wantisymmetric}--\cref{eq:wsup}}\right\} .\label{eq:dGclaim}
	\end{equation}
	We denote by $K_{2}(u)$ the set on the right side of \cref{eq:dGclaim}.
	Similarly to the previous step, one can check that $K_{2}(u)\subseteq\partial G(u)$.
	To show the opposite inclusion, we first introduce some notation.
	For every $v\in L^{2}(\mu; \Rd)$, define $\tilde{v}\in L^{2}(\mu^{\otimes2}; \Rd)$
	by $\tilde{v}(x,y)=v(x)-v(y)$, and by $F$ we denote the functional
	\cref{eq:Fdef} with the measure $\nu=\mu^{\otimes2}$. By definition,
	we have for every $v\in L^{2}(\mu; \Rd)$ that $G(v)=F(\tilde{v})$. We
	fix $p\in\partial G(u)$, so that for every $v\in L^{2}(\mu; \Rd)$, we
	have
	\[
		F(\tilde{u}+\tilde{v})\ge F(\tilde{u})+\int p\cdot v\,\dif\mu.
	\]
	Since $G$ does not change if we add a constant to its argument, it
	must be that $\int p\,\dif\mu=0$. As a consequence, we can rewrite
	the last inequality as
	\[
		F(\tilde{u}+\tilde{v})\ge F(\tilde{u})+\frac{1}{2}\int\tilde{p}\cdot\tilde{v}\,\dif\mu^{\otimes2}.
	\]
	This implies that the sets
	\begin{equation}
		\left\{ \left(\tilde{v},F(\tilde{u})+\frac{1}{2}\int\tilde{p}\cdot\tilde{v}\,\dif\mu^{\otimes2}\right)\ :\ v\in L^{2}(\mu; \Rd)\right\} \label{eq:set1}
	\end{equation}
	and
	\begin{equation}
		\left\{ \left(v',\lambda\right)\ :\ v'\in L^{2}(\mu^{\otimes2}; \Rd)\text{ and }\lambda>F(\tilde{u}+v')\right\} \label{eq:set2}
	\end{equation}
	are disjoint and convex. Moreover, the set in \cref{eq:set2} is open
	in $L^{2}(\mu^{\otimes2}; \Rd)\times\mathbf{R}$. Therefore, there is a
	hyperplane that separates the two sets. This means that there exists
	a $w\in L^{2}(\mu^{\otimes2}; \Rd)$ such that for every $v\in L^{2}(\mu; \Rd)$
	and $v'\in L^{2}(\mu^{\otimes2}; \Rd)$, we have
	\[
		F(\tilde{u}+v')-\int w\cdot v'\,\dif\mu^{\otimes2}\ge F(\tilde{u})+\frac{1}{2}\int\tilde{p}\cdot\tilde{v}\,\dif\mu^{\otimes2}-\int w\cdot\tilde{v}\,\dif\mu^{\otimes2}.
	\]
	Taking $\tilde{v}=0$, we see that $w\in\partial F(\tilde{u})$, and
	taking $v'=0$, we see that
	\[
		\int(p(x)-p(y)-2w(x,y))\cdot(v(x)-v(y))\,\dif\mu(x)\,\dif\mu(y)=0
	\]
	for all $v\in L^{2}(\mu; \Rd)$. Recalling that $\int p\,\dif\mu=0$, we
	obtain that, for $\mu$-a.e.~$x\in\mathbf{R}^{d}$,
	\[
		p(x)=\int(w(x,y)-w(y,x))\,\dif\mu(y).
	\]
	Since $w\in\partial F(\tilde{\mu})$, the result of Step 1 gives us
	that $\|w\|_{L^{\infty}}\le1$ and, for $\mu$-a.e.~$x,y\in\mathbf{R}^{d}$,
	\[
		u(x)\ne u(y)\implies w(x,y)=\sgn(u(x)-u(y)).
	\]
	We have thus completed the verification of the fact that $p\in K_{2}(u)$.

	\emph{Step 3.} It follows from the result of Step 2 that, for every
	$u\in L^{2}(\mu; \Rd)$, we have
	\[
		\partial J_{\mu,\lambda}(u)=\left\{ x\mapsto2(u(x)-x)+2\lambda\int w(x,y)\,\dif\mu(y)\ :\ w\text{ satisfies \cref{eq:wantisymmetric}--\cref{eq:wsup}}\right\} .
	\]
	In particular, since $J_{\mu,\lambda}$ is convex, a function $u\in L^{2}(\mu; \Rd)$
	is a minimizer of $J_{\mu,\lambda}$ if and only if $0\in\partial J_{\mu,\lambda}(u)$.
	Equivalently,
	\[
		J_{\mu,\lambda}(u)=\inf_{v\in L^{2}(\mu; \Rd)}J_{\mu,\lambda}(v)\iff\exists w\in L^{\infty}(\mu; \Rd)\text{ satisfying \cref{eq:wantisymmetric}--\cref{eq:xminmusux}}.
	\]
	This completes the proof of the theorem.
\end{proof}
From \cref{thm:wcharacterization}, we can prove \cref{thm:lambda1} as
a simple corollary.
\begin{proof}[Proof of \cref{thm:lambda1}.]
	By integrating \cref{eq:xminmusux} in $x$ with respect to the measure
	$\mu$, we see that $\mu$ is $\lambda$-cohesive if and only if the minimizer
	of $J_{\mu,\lambda}$ is given by $u(x)=\mathcal{C}_{\mu}(\mathbf{R}^{d})$,
	which happens if and only if there is a $w$ satisfying \cref{eq:wantisymmetric}
	and \cref{eq:wsup} such that
	\begin{equation}
		x-\mathcal{C}_{\mu}(\mathbf{R}^{d})=\lambda\int w(x,y)\,\dif\mu(y).\label{eq:constantminimizer}
	\end{equation}
	Taking $q\coloneqq\mu(\mathbf{R}^{d})\lambda w$ completes the proof.
\end{proof}

We now state a couple of lemmas which we will use to prove \cref{prop:regular}. For every $V \subset \Rd$, we write $V^{\mathrm{c}} \coloneqq \Rd \setminus V$ to denote the complement of $V$.

\begin{lem}
\label{lem:regularity-importantbit}
There is a Borel set $A\subset\R^d$ such that $\mu(\R^d\setminus A)=0$ and, for $\mu$-a.e.~$x$, we have that $V_{u_{\mu,\lambda},x}\cap A$ is $\mu$-regular and
	\begin{equation}
		\mathcal{C}_{\mu}(V_{u_{\mu,\lambda},x}\cap A)-u_{\mu,\lambda}(x)=\lambda\int_{V_{u_{\mu,\lambda},x}^{\mathrm{c}}}\sgn(u_{\mu,\lambda}(x)-u_{\mu,\lambda}(y))\,\dif\mu(y).\label{eq:averagesplit-1}
	\end{equation}
	In particular, $\mathcal{E}_{\mu,u_{\mu,\lambda}}(x)\coloneqq \mathcal{C}_{\mu}(V_{u_{\mu,\lambda},x}\cap A)$ (as in \cref{defn:regular}) is well-defined as an element of $L^\infty(\mu;\R^d)$, independently of the choice of $A$ (up to a $\mu$-null modification).
\end{lem}

\begin{proof}
	For typographical convenience, we write $u=u_{\mu,\lambda}$.
	Define
	\[
		\mathcal{E}(x)\coloneqq u(x)+\int_{V_{u,x}^\mathrm{c}}\sgn(u(x)-u(y))\,\dif\mu(y).
	\]
	Let $A \coloneqq  \{x\in\R^d\mid \mu(V_{u,x})>0\text{ or }\mathcal{E}(x)= x\}$,
	and $w$ be as in the statement of \cref{thm:wcharacterization}.
	Using~\cref{eq:sgnfordifferent}, we can rewrite \cref{eq:xminmusux}
	as, for $\mu$-a.e.~$x$,
	\begin{equation}
		x-u(x)=\lambda\int_{V_{u,x}}w(x,y)\,\dif\mu(y)+\lambda\int_{V_{u,x}^{\mathrm{c}}}\sgn(u(x)-u(y))\,\dif\mu(y).\label{eq:splitx-1}
	\end{equation}
	Since $\mathcal{E}$ is constant on each $V_{u,x}$ by definition, if $x\in A$ and $\mu(V_{u,x})=0$, then $V_{u,x}\cap A=\{x\}$ and thus \cref{eq:averagesplit-1} holds. Moreover, \cref{eq:splitx-1} implies that $\mu(\R^d\setminus A)=0$. On the other hand, if $\mu(V_{u,x})>0$, then averaging \cref{eq:splitx-1} over $x\sim \mu|_{V_{u,x}}$, we have
	\begin{align}
		\mathcal{C}_{\mu}(V_{u,x})-u(x) & =\frac{\lambda}{\mu(V_{u,x})}\iint_{V_{u,x}^{2}}w(z,y)\,\dif\mu(y)\,\dif\mu(z)\nonumber                                      \\
		                                & \qquad+\frac{\lambda}{\mu(V_{u,x})}\int_{V_{u,x}}\int_{V_{u,x}^{\mathrm{c}}}\sgn(u(z)-u(y))\,\dif\mu(y)\,\dif\mu(z)\nonumber \\
		                                & =\lambda\int_{V_{u,x}^{\mathrm{c}}}\sgn(u(x)-u(y))\,\dif\mu(y),\label{eq:averagesplit-1`}
	\end{align}
	with the second identity by \cref{eq:wantisymmetric} (to eliminate
	the first term) and the fact that $u(z)=x$ for all $z\in V_{u,x}$
	(to simplify the second term).
\end{proof}

Roughly speaking, the next lemma states that the vector formed by the centroids of two clusters and the vector formed by the values taken by the mapping $u$ on these clusters must be positively correlated. One could also say that the mapping sending each cluster centroid to the image under $u$ of any point in this cluster is a monotone operator.
\begin{lem}\label{lem:centersnotthesame}
	For $\mu$-a.e.~$x,z$ we have
	\begin{equation}
	\begin{aligned}
	\label{eq.centersnotthesame}
		(u_{\mu,\lambda}(x)&-u_{\mu,\lambda}(z))  \cdot\left(\mathcal{E}_{\mu,u_{\mu,\lambda}}(x)-\mathcal{E}_{\mu,u_{\mu,\lambda}}(z)\right)\\&\ge \lambda[\mu(V_{u_{\mu,\lambda},x})+\mu(V_{u_{\mu,\lambda},z})]|u_{\mu,\lambda}(x)-u_{\mu,\lambda}(z)|+|u_{\mu,\lambda}(x)-u_{\mu,\lambda}(z)|^2.
		\end{aligned}
	\end{equation}
\end{lem}
\begin{proof}
	For typographical convenience, let $u=u_{\mu,\lambda}$ and $\mathcal{E}=\mathcal{E}_{\mu,u_{\mu,\lambda}}$. By \cref{lem:regularity-importantbit},
	for $\mu$-a.e.~$x$
	we have
	\[
		\mathcal{E}(x)-u(x)=\lambda\int_{V_{u,x}^{\mathrm{c}}}\sgn(u(x)-u(y))\,\dif\mu(y).
	\]
	Therefore, we have for $\mu$-a.e.~$x,z$ that
	\begin{align*}
		\mathcal{E}(x)-\mathcal{E}(z) & =u(x)-u(z)+\lambda\int_{V_{u,x}^{\mathrm{c}}}\sgn(u(x)-u(y))\,\dif\mu(y)                                         \\
		                              & \qquad-\lambda\int_{V_{u,z}^{\mathrm{c}}}\sgn(u(z)-u(y))\,\dif\mu(y)                                             \\
		                              & =u(x)-u(z)+\lambda[\mu(V_{u,x})+\mu(V_{u,z})]\sgn(u(x)-u(z))                                                     \\
		                              & \qquad+\lambda\int_{(V_{u,z}\cup V_{u,z})^{\mathrm{c}}}\left[\sgn(u(x)-u(y))-\sgn(u(z)-u(y))\right]\,\dif\mu(y).
	\end{align*}
	Taking the dot product of each side with $u(x)-u(z)$, we obtain
	\begin{equation}
		\begin{aligned}(u(x)-u(z)) & \cdot\left(\mathcal{E}(x)-\mathcal{E}(z)\right)                                                                                  \\
                           & =|u(x)-u(z)|^{2}+\lambda[\mu(V_{u,x})+\mu(V_{u,z})]|u(x)-u(z)|                                                                   \\
                           & \quad+\lambda\int_{(V_{u,z}\cup V_{u,z})^{\mathrm{c}}}(u(x)-u(z))\cdot\left[\sgn(u(x)-u(z))-\sgn(u(z)-u(y))\right]\,\dif\mu(y).
		\end{aligned}
		\label{eq:takedotproducts}
	\end{equation}
	We note that for any vectors $a,b,c\in\mathbf{R}^{d}$, we have
	\begin{align*}
            (a-b)\cdot(\sgn(a-c)-&\sgn(b-c))  =\left((a-c)-(b-c)\right)\cdot\left(\frac{a-c}{|a-c|}-\frac{b-c}{|b-c|}\right) \\
		                                & =|a-c|+|b-c|-\left(\frac{1}{|a-c|}+\frac{1}{|b-c|}\right)(a-c)\cdot(b-c)       \\
		                                & \ge|a-c|+|b-c|-\left(\frac{1}{|a-c|}+\frac{1}{|b-c|}\right)|a-c||b-c| =0,
	\end{align*}
	by the Cauchy--Schwarz inequality. (If $a-c=0$ or $b-c=0$ then the inequality is still clear.) This means that the integral on
	the right side of \cref{eq:takedotproducts} is nonnegative, which implies \eqref{eq.centersnotthesame}.%
\end{proof}

\begin{proof}[Proof of \cref{prop:regular}]
	\cref{thm:wcharacterization} gives us a $w$ and a set $A\subset\R^d$ with $\mu(\R^d\setminus A)=0$ so that for all $x\in A$ such that $\mu(V_{u_{\mu,\lambda},x})=0$ we have
	\begin{align*}
		x-u_{\mu,\lambda}(x) & =\lambda\int_{V_{u_{\mu,\lambda},x}}w(x,z)\,\dif\mu(z)+\lambda\int_{V_{u_{\mu,\lambda},x}^\mathrm{c}}\sgn(u_{\mu,\lambda}(x)-u_{\mu,\lambda}(z))\,\dif\mu(z) \\
		                     & = \lambda\int\sgn(u_{\mu,\lambda}(x)-u_{\mu,\lambda}(z))\,\dif\mu(z).
	\end{align*}
	This implies that  for all $y\in V_{u_{\mu,\lambda},x}\cap A$ we must have
	\begin{align*}
		y & = u_{\mu,\lambda}(y)+\lambda\int\sgn(u_{\mu,\lambda}(y)-u_{\mu,\lambda}(z))\,\dif\mu(z)   \\
		  & =u_{\mu,\lambda}(x)+\lambda\int\sgn(u_{\mu,\lambda}(x)-u_{\mu,\lambda}(z))\,\dif\mu(z)=x.
	\end{align*}
	This proves the first condition in the definition of $\mu$-regularity. The second condition follows immediately from \cref{lem:centersnotthesame}.
\end{proof}

As a simple consequence of \cref{thm:lambda1}, we can prove the bound
\cref{eq:lambda1lb} mentioned in the introduction.
\begin{prop}
	\label{prop:lambda1lb}For any $\mu$ we have
	\begin{equation}
		\frac{R(\mu)}{\mu(\mathbf{R}^{d})}\le\lambda_{1}(\mu)\le\frac{\diam_{|\cdot|}(\supp\mu)}{\mu(\mathbf{R}^{d})}.\label{eq:lambda1lb-1}
	\end{equation}
\end{prop}

\begin{proof}
	First we show the lower bound. From \cref{thm:lambda1},
	we have a $q\colon\mathbf{R}^{d}\times\mathbf{R}^{d}\to\mathbf{R}^{d}$
	such that \cref{eq:qantisym}--\cref{eq:qcentroidcond} hold and $\|q\|_{\infty}=\lambda_{1}(\mu)\mu(\mathbf{R}^{d})$.
	Therefore, we have for $\mu$-a.e.~$x$ that
	\[
		\left|x-\mathcal{C}_{\mu}(\mathbf{R}^{d})\right|\le\fint|q(x,y)|\,\dif\mu(y)\le\|q\|_{\infty}=\lambda_{1}(\mu)\mu(\mathbf{R}^{d}),
	\]
	which implies the lower bound in \cref{eq:lambda1lb-1}. To prove the
	upper bound, let
	\begin{equation}
		q(x,y)\coloneqq x-y.\label{eq:qasdifference}
	\end{equation}
	It is obvious that $q$ satisfies \cref{eq:qantisym}--\cref{eq:qcentroidcond},
	and that $\|q\|_{\infty}=\diam_{|\cdot|}(\supp\mu)$. Therefore, \cref{thm:lambda1}
	implies the upper bound in \cref{eq:lambda1lb-1}\@.
\end{proof}

We conclude this section with the following simple proposition that allows us to replace the exact equality in \cref{eq:qcentroidcond} with an approximation.

\begin{prop}
	\label{prop:patchupthedifference}For any antisymmetric function $q_{1}\colon\mathbf{R}^{d} \times \mathbf{R}^{d}\to\mathbf{R}^{d}$,
	we have
	\begin{equation}
		\lambda_{1}(\mu)\le\mu(\mathbf{R}^{d})^{-1}\esssup_{x,y\sim\mu}\left|q_{1}(x,y)+x-y-\fint q_{1}(x,z)\,\dif\mu(z)+\fint q_{1}(y,z)\,\dif\mu(z)\right|.\label{eq:q1bd}
	\end{equation}
\end{prop}

\begin{proof}
	Let
	\[
		q(x,y)\coloneqq q_{1}(x,y)+x-y-\fint q_{1}(x,z)\,\dif\mu(z)+\fint q_{1}(y,z)\,\dif\mu(z).
	\]
	We have
	\begin{align*}
		q(y,x) & =q_{1}(y,x)+y-x-\fint q_{1}(y,z)\,\dif\mu(z)+\fint q_{1}(x,z)\,\dif\mu(z)           \\
		       & =-q_{1}(x,y)+y-x+\fint q_{1}(x,z)\,\dif\mu(z)-\fint q_{1}(z,y)\,\dif\mu(z)=-q(x,y),
	\end{align*}
	so $q$ satisfies \cref{eq:qantisym}, and moreover
	\begin{align*}
		\fint q(x,y)\,\dif\mu(y) & =\fint q(x,y)\,\dif\mu(y)+\fint x\,\dif\mu(y)-\fint y\,\dif\mu(y)-\fint q_{1}(x,z)\,\dif\mu(z) \\
		                         & \qquad+\fint\fint q_{1}(y,z)\,\dif\mu(z)\,\dif\mu(y)                                           \\
		                         & =x,
	\end{align*}
	so $q$ satisfies \cref{eq:qcentroidcond}. Thus \cref{thm:lambda1} implies
	the result.
\end{proof}

\section{Exact characterization of the clusters\label{sec:splitcondition}}

In this section, we prove \cref{thm:splitcondition,thm:chiquet-characterization,prop:ballsdontintersect}.

\begin{proof}[Proof of \cref{thm:splitcondition}.]
	We first suppose that for $\mu$-a.e.~$x$, $V_{u,x}=V_{u_{\mu,\lambda},x}$  up to a $\mu$-null set and try to prove the second statement of the theorem. Since the second statement of the theorem concerns only the level sets of $u$, we can and do assume that $u=u_{\mu,\lambda}$.
	First we show that $\mu|_{V_{u,x}}$ is cohesive for $\mu$-a.e.~$x$.

	Subtracting \cref{eq:averagesplit-1}
	from \cref{eq:splitx-1}, we have
	\[
		x-\mathcal{E}_{\mu,u}(x)=\lambda\int_{V_{u,x}}w(x,y)\,\dif\mu(y)
	\]
	for $\mu$-a.e.~$x$. By \cref{thm:wcharacterization}, this
	implies that the constant $\mathcal{E}_{\mu,u}(x)$ is a minimizer
	of $J_{\mu|_{V_{u,x}},\lambda}$, so $\mu|_{V_{u,x}}$ is $\lambda$-cohesive.

	To prove that $\mathcal{M}_{u}(\mu)$ is $\lambda$-shattered, define
	\[
		\tilde{u}(\mathcal{E}_{\mu,u}(x))\coloneqq u(x).
	\]
	This is well-defined by \cref{lem:centersnotthesame}.
	Then $\tilde{u}$ is defined $\mathcal{M}_{u}(\mu)$-a.e., and it
	is clear that $\tilde{u}$ can be extended to an injection on $\Rd$. By \cref{eq:averagesplit-1}
	we have
	\[
		\tilde{u}(X)=X-\lambda\int\sgn(\tilde{u}(X)-\tilde{u}(Y))\,\dif\mathcal{M}_{u}(\mu)(Y)
	\]
	for $\mathcal{M}_{u}(\mu)$-a.e.~$X$. Taking $\tilde{w}(X,Y)=\sgn(X-Y)$
	as the $w$ in \cref{thm:wcharacterization}, we see that $\tilde{u}$
	is in fact a minimizer of $J_{\mathcal{M}_{u}(\mu),\lambda}$. Thus
	$\mathcal{M}_{u}(\mu)$ is $\lambda$-shattered.

	Now we prove the other direction, so suppose we have a $\mu$-regular function $u$ such that
	$\mathcal{M}_{u}(\mu)$ is $\lambda$-shattered and, for $\mu$-a.e.~$x$,
	the restriction $\mu|_{V_{u,x}}$ is $\lambda$-cohesive. Let $\tilde{u}$
	be the (injective) minimizer of $J_{\mathcal{M}_{u}(\mu),\lambda}$
	and define
	\begin{equation}
		v(x)=\tilde{u}(\mathcal{E}_{\mu,u}(x)),\label{eq:vdef}
	\end{equation}
	noting that the assumption that $u$ is $\mu$-regular means that $\mathcal{E}_{\mu,u}$ is defined.
	Since $\tilde{u}$ is injective, we see that $v$ has the same level
	sets as $u$. We want to prove that $v$ is a minimizer of $J_{\mu,\lambda}$.
	For $\mu$-a.e.\ $x$, by \cref{thm:wcharacterization} and the fact
	that $\mu|_{V_{u,x}}$ is $\lambda$-cohesive, we have an antisymmetric
	$w_{V_{u,x}}$, bounded in norm by $1$, such that
	\begin{equation}
		x-\mathcal{E}_{\mu,u}(x)=\lambda\int_{V_{u,x}}w_{V_{u,x}}(x,y)\,\dif\mu(y).\label{eq:betweencentroids}
	\end{equation}
	Moreover, using \cref{eq:vdef} and \cref{eq:xminmusux} we have
	\begin{align}
		\mathcal{E}_{\mu,u}(x)-v(x) & =\lambda\int\sgn(\tilde{u}(\mathcal{E}_{\mu,u}(x))-\tilde{u}(\mathcal{E}_{\mu,u}(y)))\,\dif\mathcal{M}_{u}(\mu)(y)\nonumber \\
		                            & =\lambda\int_{V_{u,x}^{\mathrm{c}}}\sgn(v(x)-v(y))\,\dif\mu(y).\label{eq:fromcentroid}
	\end{align}
	So define
	\[
		w(x,y)=\begin{cases}
			w_{V_{u,x}}(x,y) & \text{ if } u(x)=u(y);    \\
			\sgn(v(x)-v(y))  & \text{ if } u(x)\ne u(y).
		\end{cases}
	\]
	Then we have, using \cref{eq:betweencentroids} and \cref{eq:fromcentroid},
	that
	\begin{align*}
		x-v(x) & =x-\mathcal{E}_{\mu,u}(x)+\mathcal{E}_{\mu,u}(x)-v(x)                                                            \\
		       & =\lambda\int_{V_{u,x}}w_{V_{u,x}}(x,y)\,\dif\mu(y)+\lambda\int_{V_{u,x}^{\mathrm{c}}}\sgn(v(x)-v(y))\,\dif\mu(y) \\
		       & =\lambda\int w(x,y)\,\dif\mu(y),
	\end{align*}
	verifying \cref{eq:xminmusux}. Conditions~\cref{eq:wantisymmetric}--\cref{eq:wsup}
	are clearly satisfied for $w$, so this proves that $v$ is a minimizer
	of $J_{\mu,\lambda}$.
\end{proof}
Now we give a proof of \cref{thm:chiquet-characterization} in our setting.
The key ingredient is the following proposition proved in the discrete
case by \citet{CGR17}.
\begin{prop}
	\label{prop:concentrateonV}Fix a Borel set $A\subset\mathbf{R}^{d}$ with $\mu(A)>0$
	and assume that $\mu|_{A}$ is $\lambda$-cohesive. Define
	\[
		u(x)\coloneqq \begin{cases}
			\mathcal{C}_{\mu}(A) & \text{ if } x\in A;    \\
			x                    & \text{ if }x\not\in A.
		\end{cases}
	\]
	Thus $\mathcal{M}_{u}(\mu)$ is the measure obtained from $\mu$ by
	consolidating all of the mass in $A$ at $\mathcal{C}_{\mu}(A)$.
	Then, for $\mu$-a.e.~$x$, we have
	\begin{equation}
		u_{\mu,\lambda}(x)=u_{\mathcal{M}_{u}(\mu),\lambda}(u(x)).\label{eq:factorsthroughcohesive}
	\end{equation}
\end{prop}

\begin{proof}
	We follow the argument given by \citet[proof of Theorem 1(b)]{JVZ19}.
	We apply \cref{thm:wcharacterization} twice. First, by \cref{thm:wcharacterization}
	applied to $\mathcal{M}_{u}(\mu)$, there is an antisymmetric, $1$-bounded
	$w_{\mathrm{out}}\in L^{\infty}(\mathcal{M}_{u}(\mu)^{\otimes2}; \Rd)$
	satisfying
	\begin{equation}
		u_{\mathcal{M}_{u}(\mu),\lambda}(x)\ne u_{\mathcal{M}_{u}(\mu),\lambda}(y)\implies w_{\mathrm{out}}(x,y)=\sgn(u_{\mathcal{M}_{u}(\mu),\lambda}(x)-u_{\mathcal{M}_{u}(\mu),\lambda}(y))\label{eq:notthesame-Mumu}
	\end{equation}
	and
	\[
		x-u_{\mathcal{M}_{u}(\mu),\lambda}(x)=\lambda\int w_{\mathrm{out}}(x,z)\,\dif\mathcal{M}_{u}(\mu)(z)
	\]
	for $\mu$-a.e.~$x,y$. Second, by \cref{thm:wcharacterization} applied
	to $\mu|_{A}$, there is an antisymmetric, $1$-bounded $w_{\mathrm{in}}\in L^{\infty}((\mu|_{A})^{\otimes2}; \Rd)$
	satisfying
	\[
		x-\mathcal{C}_{\mu}(A)=\lambda\int_{A}w_{\mathrm{in}}(x,z)\,\dif\mu(z)
	\]
	for $\mu$-a.e.~$x\in A$.

	Now define
	\[
		w(x,y)=\begin{cases}
			w_{\mathrm{in}}(x,y)        & \text{ if } x,y\in A; \\
			w_{\mathrm{out}}(u(x),u(y)) & \text{ otherwise}.
		\end{cases}
	\]
	It is clear that $w$ is antisymmetric and $1$-bounded since $w_{\mathrm{in}}$
	and $w_{\mathrm{out}}$ are. It is also clear from \cref{eq:notthesame-Mumu}
	that if $u_{\mathcal{M}_{u}(\mu),\lambda}(u(x))\ne u_{\mathcal{M}_{u}(\mu),\lambda}(u(y))$
	then $w(x,y)=\sgn(u_{\mathcal{M}_{u}(\mu),\lambda}(u(x))-u_{\mathcal{M}_{u}(\mu),\lambda}(u(y)))$.
	For $\mu$-a.e.~$x\in A$, we have
	\begin{align*}
		\lambda\int w(x,z)\,\dif\mu(z) & =\lambda\int_{A}w(x,z)\,\dif\mu(z)+\lambda\int_{A^{\mathrm{c}}}w(x,z)\,\dif\mu(z)                                                                  \\
		                               & =\lambda\int_{A}w_{\mathrm{in}}(x,z)\,\dif\mu(z)+\lambda\int_{A^{\mathrm{c}}}w_{\mathrm{out}}(\mathcal{C}_{\mu}(A),z)\,\dif\mathcal{M}_{u}(\mu)(z) \\
		                               & =x-\mathcal{C}_{\mu}(A)+\mathcal{C}_{\mu}(A)-u_{\mathcal{M}_{u}(\mu),\lambda}(\mathcal{C}_{\mu}(A))                                                \\
		                               & =x-u_{\mathcal{M}_{u}(\mu),\lambda}(u(x)),
	\end{align*}
	while for $\mu$-a.e.~$x\not\in A$ we have
	\begin{align*}
		\lambda\int w(x,z)\,\dif\mu(z) & =\lambda\int w_{\mathrm{out}}(x,u(z))\,\dif\mu(z)                                                                                  \\
		                               & =\lambda\int_{A}w_{\mathrm{out}}(x,\mathcal{C}_{\mu}(A))\,\dif\mu(z)+\lambda\int_{A^{\mathrm{c}}}w_{\mathrm{out}}(x,z)\,\dif\mu(z) \\
		                               & =\lambda\int w_{\mathrm{out}}(x,z)\,\dif\mathcal{M}_{u}(\mu)(z)                                                                    \\
		                               & =x-u_{\mathcal{M}_{u}(\mu),\lambda}(u(x)).
	\end{align*}
	Then \cref{eq:factorsthroughcohesive} follows from \cref{thm:wcharacterization}.
\end{proof}
\begin{proof}[Proof of \cref{thm:chiquet-characterization}.]
	\cref{thm:splitcondition} implies that any level set of $u_{\mu,\lambda}$
	is $\lambda$-cohesive, and \cref{prop:concentrateonV} implies that
	$\mu(A)>0$ and $\mu|_{A}$ is $\lambda$-cohesive then $A$ is contained in a
	single level set of $u_{\mu,\lambda}$. These two facts together imply the statement of the theorem.
\end{proof}

\begin{proof}[Proof of \cref{thm:agglomeration}]
We fix $\lambda \le \lambda'$. We first confirm that if a measure $\mu$ is $\lambda$-cohesive, then it is $\lambda'$-cohesive. Indeed, if $\mu$ is $\lambda$-cohesive, then there exists a constant $c \in \Rd$ such that for every $u \in L^2(\mu;\Rd)$, 
\begin{equation*}  %
J_{\mu,\lambda}(c) \le J_{\mu,\lambda}(u). 
\end{equation*}
Since $\lambda' \ge \lambda$, we deduce that 
\begin{equation*}  %
J_{\mu,\lambda'}(c) = J_{\mu,\lambda}(c) \le J_{\mu,\lambda}(u) \le J_{\mu,\lambda'}(u). 
\end{equation*}
This shows that the constant $c$ minimizes $J_{\mu,\lambda'}$, and by uniqueness of the minimizer, that $\mu$ is $\lambda'$-cohesive.

The proof of \cref{thm:agglomeration} is now an application of \cref{thm:chiquet-characterization}. Outside a set of $\mu$-measure zero, every $x \in \Rd$ satisfies the statement of this theorem both for $\lambda$ and for $\lambda'$. For each such $x \in \Rd$, the measure $\mu|_{V_{u_{\mu,\lambda},x}}$ is $\lambda$-cohesive, so by the previous observation, it is $\lambda'$-cohesive. Applying the second part of \cref{thm:chiquet-characterization} with $\lambda'$, we deduce that $\mu(V_{u_{\mu,\lambda},x} \setminus V_{u_{\mu,\lambda'},x}) = 0$, as desired.
\end{proof}

Finally, we prove \cref{prop:ballsdontintersect}.
\begin{proof}[Proof of \cref{prop:ballsdontintersect}.]
By Theorem~\ref{thm:splitcondition}, the measure $\mu|_{V_{u_{\mu,\lambda},x}}$ is $\lambda$-cohesive. By \cref{prop:lambda1lb}, we must therefore have that
\begin{equation*}  %
\lambda \ge \lambda_1\Ll(\mu|_{V_{u_{\mu,\lambda},x}}\Rr) \ge \frac{R(\mu|_{V_{u_{\mu,\lambda},x}})}{\mu|_{V_{u_{\mu,\lambda},x}}(\Rd)}. 
\end{equation*}
Rearranging, we obtain \eqref{e.clusterball}.

We now turn to \cref{eq:centroidsfarapart}. 
By \cref{thm:splitcondition}, the measure $\mathcal{M}_{u_{\mu,\lambda}}(\mu)=(\mathcal{E}_{\mu,u_{\mu,\lambda}})_*(\mu)$ is $\lambda$-shattered. Then \cref{lem:centersnotthesame} implies that, for $\mathcal{M}_{u_{\mu,\lambda}}(\mu)$-a.e.\ $x,z$ with $x\ne z$, we have
\[
|x-z|> \lambda[\mu(V_{u_{\mu,\lambda},x})+\mu(V_{u_{\mu,\lambda},z})].
\]
This yields \cref{eq:centroidsfarapart} for $\mu$-a.e.\  $x,z$ with $u_{\mu,\lambda}(x)\ne u_{\mu,\lambda}(z)$.
\end{proof}

\section{Stability properties\label{sec:stability}}

In this section, we prove some stability results for $\lambda_{1}(\mu)$
and $\lambda_{*}(\mu)$. For this purpose, we introduce some definitions related to optimal transport. Let $\mu, \td \mu$ be finite measures of compact support such that $\mu(\Rd) = \td \mu(\Rd)$. We denote by $\Gamma(\mu,\td \mu)$ the set of Borel measures on $\Rd \times \Rd$ whose first marginal is $\mu(\Rd) \, \mu(\cdot)$ and second marginal is $\mu(\Rd) \, \td \mu(\cdot)$. For $p \in [1,\infty)$, the $p$-Wasserstein distance between $\mu$ and $\td \mu$ is
\begin{equation*}  %
	\mcl W_p(\mu, \td \mu) \coloneqq \Ll( \inf_{\pi \in \Gamma(\mu,\td \mu)} \int |x-\td x|^p \,  \d \pi(x,\td x) \Rr)^\frac 1 p,
\end{equation*}
while
\begin{equation*}  %
	\mcl W_\infty(\mu, \td \mu) \coloneqq \adjustlimits\inf_{\pi \in \Gamma(\mu,\td \mu)} \esssup_{(x,\td x) \sim \pi} |x-\td x|.
\end{equation*}
It is classical to show that for each $p \in [1,\infty]$, this problem admits an optimizer in $\Gamma(\mu,\td \mu)$. We call any optimizer a $p$-optimal transport plan from $\mu$ to $\td \mu$. At least when $p < \infty$ and if the measure $\mu$ is absolutely continuous with respect to the Lebesgue measure, there in fact exists a measurable mapping $T \colon \Rd \to \Rd$ such that the image of the measure $\mu$ by the mapping $(\mathrm{Id}, T)$ is an optimal transport plan from $\mu$ to $\td \mu$. In such a case, we call the mapping $T$ an optimal transport map from $\mu$ to $\td \mu$. In this paper, we will only make use of optimal transport maps for $p = 1$. In this case, a proof of existence can be found in \citet[Theorem~6.2]{ambrosio}.

\subsection{Stability of \texorpdfstring{$\lambda_{1}$}{λ₁}}

In this section we prove two stability results for $\lambda_{1}(\mu)$.
The first is that $\lambda_{1}(\mu)$ is continuous under absolutely
continuous perturbations of $\mu$. As is standard in measure theory, for measures
$\mu$ and $\tilde\mu$, we write $\tilde\mu\ll\mu$ to mean that $\tilde\mu$ is absolutely continuous with respect to $\mu$.
\begin{prop}[Absolutely continuous perturbations]
	\label{prop:stability-AC}Suppose that $\eps<1$ and $\tilde{\mu}$
	and $\mu$ are finite measures such that $\tilde{\mu}\ll\mu$,
	\[
		\left|\frac{\dif\tilde{\mu}}{\dif\mu}(z)-1\right|<\eps,
	\]
	and
	\[
		\tilde{\mu}(\mathbf{R}^{d})\ge(1-\eps)\mu(\mathbf{R}^{d}).
	\]
	Then
	\begin{equation}
		\lambda_{1}(\tilde{\mu})\le\frac{1+2\eps}{1-\eps}\lambda_{1}(\mu).\label{eq:lambda1comp}
	\end{equation}
\end{prop}

The second stability result says that $\lambda_{1}(\mu)$ is continuous
under $\mathcal{W}_{\infty}$ perturbations of $\mu$:
\begin{prop}[$\mathcal{W}_{\infty}$ perturbations]
	\label{prop:Winfty-perturbations}Let $\tilde{\mu}$ and
	$\mu$ be finite measures of compact support such that $\mu(\mathbf{R}^{d})=\tilde{\mu}(\mathbf{R}^{d})$. Then we have
	\begin{equation}
		\left|\lambda_{1}(\tilde{\mu})-\lambda_{1}(\mu)\right|\le\frac{3\mathcal{W}_{\infty}(\mu,\tilde{\mu})}{\mu(\mathbf{R}^{d})}.\label{eq:Winftybound}
	\end{equation}
\end{prop}

Now we prove the two preceding propositions.
\begin{proof}[Proof of \cref{prop:stability-AC}.]
	Let $q$ satisfying \cref{eq:qantisym}--\cref{eq:qcentroidcond}
	(for $\mu$) be such that
	\[
		\|q\|_{\infty}=\lambda_{1}(\mu)\mu(\mathbf{R}^{d}).
	\]
	Then by \cref{prop:patchupthedifference} we have
	\begin{align*}
		\lambda_{1}(\tilde{\mu}) & \le\tilde{\mu}(\mathbf{R}^{d})^{-1}\esssup_{x,y\sim\mu}\left|q(x,y)+x-y-\fint q(x,z)\,\dif\tilde{\mu}(z)+\fint q(y,z)\,\dif\tilde{\mu}(z)\right|            \\
		                         & =\tilde{\mu}(\mathbf{R}^{d})^{-1}\esssup_{x,y\sim\mu}\left|q(x,y)+x-y-\fint (q(x,z)-q(y,z))\frac{\dif\tilde{\mu}}{\dif\mu}(z)\,\dif\mu(z)\right|            \\
		                         & =\tilde{\mu}(\mathbf{R}^{d})^{-1}\esssup_{x,y\sim\mu}\left|q(x,y)-\fint (q(x,z)-q(y,z))\left(\frac{\dif\tilde{\mu}}{\dif\mu}(z)-1\right)\,\dif\mu(z)\right| \\
		                         & \le(1+2\eps)\tilde{\mu}(\mathbf{R}^{d})^{-1}\|q\|_{\infty}                                                                                                  \\
		                         & \le\frac{1+2\eps}{1-\eps}\lambda_{1}(\mu),
	\end{align*}
	as announced.
\end{proof}
\begin{proof}[Proof of \cref{prop:Winfty-perturbations}.]
	Let $q$ satisfying \cref{eq:qantisym}--\cref{eq:qcentroidcond}
	(for $\mu$) be such that
	\[
		\|q\|_{\infty}=\lambda_{1}(\mu)\mu(\mathbf{R}^{d}).
	\]
	Let $\pi$ be an $\infty$-optimal transport plan from $\tilde{\mu}$
	to $\mu$. We write the disintegration \cite[Section~I.4]{PanLec}
	\[
		\dif\pi(x,x')=\dif\nu(x'\mid x)\dif\tilde{\mu}(x).
	\]
	Define
	\[
		q_{1}(x,y)\coloneqq\iint q(w,z)\,\dif\nu(z\mid y)\,\dif\nu(w\mid x),
	\]
	which is antisymmetric by Fubini's theorem. We note that
	\[
		\|q_{1}\|_{\infty}\le\|q\|_{\infty}=\lambda_{1}(\mu)\mu(\mathbf{R}^{d}).
	\]
	We also have
	\begin{align*}
		\fint q_{1}(x,y)\,\dif\tilde{\mu}(y) & =\frac{1}{\tilde{\mu}(\mathbf{R}^{d})}\iiint q(w,z)\,\dif\nu(z\mid y)\,\dif\nu(w\mid x)\,\dif\tilde{\mu}(y) \\
		                                     & =\frac{1}{\tilde{\mu}(\mathbf{R}^{d})}\iiint q(w,z)\,\dif\nu(w\mid x)\,\dif\pi(y,z)                         \\
		                                     & =\frac{1}{\mu(\mathbf{R}^{d})}\iint q(w,z)\,\dif\mu(z)\,\dif\nu(w\mid x)                                    \\
		                                     & =\int w\,\dif\nu(w\mid x)-\mathcal{C}_{\mu}(\mathbf{R}^{d}),
	\end{align*}
	with the last identity by \cref{eq:qcentroidcond}. Thus we have
	\[
		\left|\fint q_{1}(x,y)\,\dif\tilde{\mu}(y)-[x-\mathcal{C}_{\mu}(\mathbf{R}^{d})]\right|\le\mathcal{W}_{\infty}(\mu,\tilde{\mu}).
	\]
	Therefore, we have by \cref{prop:patchupthedifference} that
	\begin{align*}
		\lambda_{1}(\tilde{\mu}) & \le\tilde{\mu}(\mathbf{R}^{d})^{-1}\esssup_{x,y\sim\mu}\left|q_{1}(x,y)+x-y-\fint q_{1}(x,z)\,\dif\tilde{\mu}(z)+\fint q_{1}(y,z)\,\dif\tilde{\mu}(z)\right|                               \\
		                         & \le\tilde{\mu}(\mathbf{R}^{d})^{-1}\bigg[\esssup_{x,y\sim\mu}\left(\left|q_{1}(x,y)\right|+2\left|\fint q_{1}(x,z)\,\dif\tilde{\mu}(z)-[x-\mathcal{C}_{\mu}(\mathbf{R}^{d})]\right|\right) \\&\qquad\qquad\qquad\qquad+\left|\mathcal{C}_{\mu}(\mathbf{R}^{d})-\mathcal{C}_{\tilde{\mu}}(\mathbf{R}^{d})\right|\bigg]\\
		                         & \le\tilde{\mu}(\mathbf{R}^{d})^{-1}\left(\lambda_{1}(\mu)\mu(\mathbf{R}^{d})+3\mathcal{W}_{\infty}(\mu,\tilde{\mu})\right).
	\end{align*}
	By the symmetry between $\mu$ and $\tilde{\mu}$, this yields \cref{eq:Winftybound}.
\end{proof}

\subsection{Stability of \texorpdfstring{$\lambda_{*}$}{λ*}}

We now show that, for atomic measures, $\lambda_{*}$ is stable under
$\mathcal{W}_{1}$ perturbation of the measures. The key ingredient
will be the following continuity property.
\begin{prop}
	\label{prop:W1-perturbations}Let $\lambda>0$, $M\in(0,\infty)$,
	and let $\mu,\tilde{\mu}$ be two Borel probability measures on~$\mathbf{R}^{d}$
	such that $\supp\mu,\supp\tilde{\mu}\subset B_{M}(0)$.
	\begin{enumerate}
		\item \label{enu:disintegratedversion}For every $1$-optimal transport plan
		      $\pi$ from $\mu$ to $\tilde{\mu}$, denoting its disintegration by
		      \[
			      \dif\pi(x,\tilde{x})=\dif\nu(\tilde{x}\mid x)\,\dif\mu(x),
		      \]
		      we have
		      \begin{equation}
			      \int\left|u_{\mu,\lambda}(x)-\int u_{\tilde{\mu},\lambda}(\tilde x)\,\dif\nu(\tilde{x}\mid x)\right|^{2}\,\dif\mu(x)\le16M\mathcal{W}_{1}(\mu,\tilde{\mu}).\label{eq:cont-mu}
		      \end{equation}
		\item \label{enu:integratedversion}There exists a $1$-optimal transport plan
		      $\pi$ from $\mu$ to $\td{\mu}$
		      such that
		      \[
			      \int|u_{\mu,\lambda}(x)-u_{\tilde{\mu},\lambda}(\tilde{x})|^{2}\,\dif\pi(x,\tilde x)\le16M\mathcal{W}_{1}(\mu,\tilde{\mu}).
		      \]
	\end{enumerate}
\end{prop}

\begin{proof}
	We start with part~\eqref{enu:disintegratedversion}. For $\mu$-a.e.~$x\in\mathbf{R}^{d}$,
	we put
	\[
		\overline{u}(x)\coloneqq\int u_{\tilde\mu,\lambda}(\tilde{x})\,\dif\nu(\tilde{x}\mid x).
	\]
	We then observe that
	\begin{align*}
		\inf J_{\tilde{\mu},\lambda} & =\int|u_{\tilde{\mu},\lambda}(\tilde{x})-\tilde{x}|^{2}\,\dif\tilde{\mu}(\tilde{x})+\lambda\iint|u_{\tilde{\mu},\lambda}(\tilde{y})-u_{\tilde{\mu},\lambda}(\tilde{x})|\,\dif\tilde{\mu}(\tilde{x})\,\dif\tilde{\mu}(\tilde{y})                        \\
		                             & \ge\int|u_{\tilde{\mu},\lambda}(\tilde{x})-x|^{2}\,\dif\pi(x,\tilde{x})+\lambda\iint|u_{\tilde{\mu},\lambda}(\tilde{y})-u_{\tilde{\mu},\lambda}(\tilde{x})|\,\dif\tilde{\mu}(\tilde{x})\,\dif\tilde{\mu}(\tilde{y})-4M\mathcal{W}_{1}(\mu,\tilde{\mu}) \\
		                             & \ge\int|\overline{u}(x)-x|^{2}\,\dif\mu(x)+\lambda\iint|\overline{u}(y)-\overline{u}(x)|\,\dif\mu(x)\,\dif\mu(y)-4M\mathcal{W}_{1}(\mu,\tilde{\mu}),
	\end{align*}
	where we used the disintegration of $\pi$ and Jensen's inequality
	in the last step. We can rewrite this as
	\begin{equation}
		\inf J_{\mu,\lambda}\le J_{\mu,\lambda}(\overline{u})\le\inf J_{\tilde{\mu},\lambda}+4M\mathcal{W}_{1}(\mu,\tilde{\mu}).\label{eq:compare-tdu}
	\end{equation}
	By symmetry, we conclude that
	\begin{equation}
		\Ll|\inf J_{\mu,\lambda}-\inf J_{\tilde{\mu},\lambda}\Rr| \le 4 M \mathcal{W}_{1}(\mu,\tilde{\mu}).\label{eq:lipW}
	\end{equation}
	Using \cref{eq:convexity} and then \cref{eq:compare-tdu}, we thus deduce
	that
	\begin{align*}
		\frac{1}{4}\int|u_{\mu,\lambda}-\overline{u}|^2\,\dif\mu & \le\frac{1}{2}(J_{\mu,\lambda}(u_{\mu,\lambda})+J_{\mu,\lambda}(\overline{u}))-J_{\mu,\lambda}\left(\frac{u_{\mu,\lambda}+\overline{u}}{2}\right) \\
		                                                       & \le\frac{1}{2}\left(\inf J_{\tilde{\mu},\lambda}-\inf J_{\mu,\lambda}\right)+2M\mathcal{W}_{1}(\mu,\tilde{\mu}).
	\end{align*}
	Combining this with \cref{eq:lipW}, we obtain \cref{eq:cont-mu}.

	We now turn to the proof of part~\eqref{enu:integratedversion} of the proposition. We argue by approximation.
	For every $\eps>0$, we let $\mu_{\eps}$ be a measure on $B_{M}(0)$
	that is absolutely continuous with respect to the Lebesgue measure
	and such that
	\begin{equation}
		\mathcal{W}_{1}(\mu,\mu_{\eps})\le\eps.\label{eq:W1eps}
	\end{equation}
	We denote by $T_{\eps}$ and $\tilde{T}_{\eps}$ $1$-optimal transport
	maps from $\mu_{\eps}$ to $\mu$ and from $\mu_{\eps}$ to $\tilde{\mu}$,
	respectively. We have, for every $\delta>0$, that
	\begin{align*}
                &\int  |u_{\mu,\lambda}(T_{\eps}(x))-u_{\tilde{\mu},\lambda}(\tilde{T}_{\eps}(x))|^{2}\,\dif\mu_{\eps}(x)                                                                                                                   \\
                & \qquad\le(1+\delta^{-1})\int|u_{\mu,\lambda}(T_{\eps}(x))-u_{\mu_{\eps},\lambda}(x)|^{2}\,\dif\mu_{\eps}(x)\\&\qquad\qquad+(1+\delta)\int|u_{\mu_{\eps},\lambda}(x)-u_{\tilde{\mu},\lambda}(\tilde{T}_{\eps}(x))|^{2}\,\dif\mu_{\eps}(x).
	\end{align*}
	Using part~\eqref{enu:disintegratedversion} of the proposition
	and \cref{eq:W1eps}, we deduce that
	\[
		\int|u_{\mu,\lambda}(T_{\eps}(x))-u_{\tilde{\mu},\lambda}(\tilde{T}_{\eps}(x))|^{2}\,\dif\mu_{\eps}(x)\le16M^{2}(1+\delta^{-1})\eps+16M(1+\delta)\mathcal{W}_{1}(\mu_{\eps},\tilde{\mu}).
	\]
	The image of the measure $\mu_{\eps}$ under the mapping $(T_{\eps},\tilde{T}_{\eps})$
	is a coupling between the measures $\mu$ and $\tilde{\mu}$. Up to
	the extraction of a subsequence, we can assume that this image measure
	converges to a coupling $\pi$ as $\eps\downarrow0$. Using \cref{eq:W1eps}
	once more, we thus have that
	\[
		\int|u_{\mu,\lambda}(x)-u_{\tilde{\mu},\lambda}(\tilde{x})|^{2}\,\dif\pi(x,\tilde{x})\le16M(1+\delta)\mathcal{W}_{1}(\mu,\tilde{\mu}).
	\]
	Since $\delta>0$ was arbitrary, the factor $1+\delta$ on the right
	side can be removed. In order to conclude, we must show that $\pi$
	is an optimal transport plan. This follows from a similar line of
	reasoning: we have
	\begin{align*}
		\int|T_{\eps}(x)-\tilde{T}_{\eps}(\tilde{x})|\,\dif\mu_{\eps}(x) & \le\int|T_{\eps}(x)-x|\,\dif\mu_{\eps}(x)+\int|x-\tilde{T}_{\eps}(x)|\,\dif\mu_{\eps}(x) \\
		                                                                 & \le\eps+\mathcal{W}_{1}(\mu_{\eps},\tilde{\mu}),
	\end{align*}
	so that, upon passing to the limit $\eps\downarrow0$, we get
	\[
		\int|x-\tilde{x}|\,\dif\pi(x,\tilde{x})\le\mathcal{W}_{1}(\mu,\tilde{\mu}),
	\]
	as desired.
\end{proof}
\begin{prop}
	\label{prop:close-shattered}Let $M\in(0,\infty)$ and suppose that
	$\mu$ and $\tilde{\mu}$ are finite, purely atomic probability measures
	with support in $B_{M}(0)$. Suppose also that $\mu$ is $\lambda$-shattered,
	which means that $u_{\mu,\lambda}$ is injective on $\supp\mu$. Define
	\[
            \delta_{1}=\essinf_{\substack{(x,y)\sim\mu^{\otimes2}\\x\ne y}}|u_{\mu,\lambda}(x)-u_{\mu,\lambda}(y)|\qquad\text{ and }\qquad\delta_{2}=\essinf_{x\sim\mu}\tilde{\mu}(\{x\}).
	\]
	If
	\begin{equation}
		\mathcal{W}_{1}(\mu,\tilde{\mu})<\frac{\delta_{1}^{2}\delta_{2}}{32M},\label{eq:W1ub}
	\end{equation}
	then $\tilde{\mu}$ is also $\lambda$-shattered.
\end{prop}

\begin{proof}
	By \cref{prop:W1-perturbations}, there is a $1$-optimal
	transport plan $\pi$ from $\mu$ to $\tilde{\mu}$ such that
	\begin{equation}
		\int|u_{\mu,\lambda}(x)-u_{\tilde{\mu},\lambda}(\tilde{x})|^{2}\,\dif\pi(x,\tilde{x})\le16M\mathcal{W}_{1}(\mu,\tilde{\mu}).\label{eq:applyW1perturb}
	\end{equation}
	Suppose there are distinct points $\tilde{x}_{1},\tilde{x}_{2}\in\supp\tilde{\mu}$
	(a finite set) such that $u_{\tilde{\mu},\lambda}(\tilde{x}_{1})=u_{\tilde{\mu},\lambda}(\tilde{x}_{2})$.
	Then we have by the triangle inequality that
	\[
		|u_{\mu,\lambda}(x_{1})-u_{\tilde{\mu},\lambda}(\tilde{x}_{1})|^{2}+|u_{\mu,\lambda}(x_{2})-u_{\tilde{\mu},\lambda}(\tilde{x}_{2})|^{2}\ge\frac{1}{2}|u_{\mu,\lambda}(x_{1})-u_{\mu,\lambda}(x_{2})|^{2}\ge\delta_1^{2}/2.
	\]
	Denote the disintegration of $\pi$ over the first coordinate by
	\[
		\dif\pi(x,\tilde{x})=\dif\tilde{\nu}(x\mid\tilde{x})\dif\tilde{\mu}(\tilde{x}).
	\]
	Then we have
	\begin{align*}
            \delta_{1}^{2}&\delta_{2}  \le\frac{1}{2}\delta_{1}^{2}\left(\tilde{\mu}(\{x_{1}\})+\tilde{\mu}(\{x_{2}\})\right)                                                                                                                                                                                                            \\
		                         & \le\int_{\tilde{x}\in\{\tilde{x}_{1},\tilde{x}_{2}\}}\iint[|u_{\mu,\lambda}(x_{1})-u_{\tilde{\mu},\lambda}(\tilde{x})|^{2}+|u_{\mu,\lambda}(x_{2})-u_{\tilde{\mu},\lambda}(\tilde{x})|^{2}]\,\dif\tilde{\nu}(x_{1}\mid\tilde{x})\,\dif\tilde{\nu}(x_{2}\mid\tilde{x})\,\dif\tilde{\mu}(\tilde{x}) \\
		                         & =2\int_{(x,\tilde{x})\in\mathbf{R}^{d}\times\{\tilde{x}_{1},\tilde{x}_{2}\}}|u_{\mu,\lambda}(x)-u_{\tilde{\mu},\lambda}(\tilde{x})|^{2}\,\dif\pi(x,\tilde{x})                                                                                                                                     \\
		                         & \le32M\mathcal{W}_{1}(\mu,\tilde{\mu}),
	\end{align*}
	with the last inequality by \cref{eq:applyW1perturb}. But this contradicts
	\cref{eq:W1ub}. Therefore, $u_{\tilde{\mu},\lambda}$ must be injective
	on $\supp\tilde{\mu}$. This means that $\tilde{\mu}$ is $\lambda$-shattered.
\end{proof}

\subsection{Proofs of Theorems~\ref{thm:stability} and~\ref{t.stoch.ball}\label{subsec:thmstabilityproof}}

Now we can prove our main stability results, \cref{thm:stability,t.stoch.ball}.
\begin{proof}[Proof of \cref{thm:stability}]
	For $i\in\{1,\ldots,I\}$, define
	\[
		q_{i,N}=\#\{n\in\{1,\ldots,N\}\mid X_{n}\in\overline{U_{i}}\}.
	\]
	By the law of large numbers, we have with probability $1$ that
	\begin{equation}\label{eq:LLN}
		\lim_{N\to\infty}N^{-1}q_{i,N}=\mu(\overline{U_{i}}).
	\end{equation}
	Define
	\[
		\tilde{\mu}_{N,i}=\frac{1}{q_{i,N}}\mu_{N}|_{\overline{U_{i}}}.%
	\]
	By \cref{eq:LLN} and Theorem~1.1 of \citet{GS15} for $d\ge2$, or a similar result using the Glivenko--Cantelli theorem \cite[Theorem~2.4.7]{durrettbook}
	for $d=1$, we have that
	\[
		\tilde{\mu}_{N,i}\to\frac{1}{\mu(\overline{U_{i}})}\mu|_{\overline{U_{i}}}
	\]
	in probability as $N\to\infty$ with respect to the $\mathcal{W}^{\infty}$
	topology. Therefore, we have that
	\[
		\lim_{N\to\infty}\lambda_{1}(\tilde{\mu}_{N,i})=\lambda_{1}(\mu|_{\overline{U_{i}}})
	\]
	in probability by \cref{prop:Winfty-perturbations}. On the other hand,
	we have that
	\[
		\lim_{N\to\infty}|\lambda_{1}(\tilde{\mu}_{N,i})-\lambda_{1}(\mu_{N}|_{\overline{U}_{i}})|=0
	\]
	in probability by \cref{prop:stability-AC}. Combining the last two
	displays, we see that
	\begin{equation}
		\lambda_{1}(\mu_{N}|_{\overline{U}_{i}})\to\lambda_{1}(\mu|_{\overline{U_{i}}})\label{eq:lambda1conv}
	\end{equation}
	as $N\to\infty$. On the other hand, it is clear from the law of large
	numbers that
	\[
		\lim_{N\to\infty}\mathcal{M}_{u}(\mu_{N})=\mathcal{M}_{u}(\mu)
	\]
	in probability with respect to the $\mathcal{W}^{1}$ topology. Therefore,
	we have from \cref{prop:close-shattered} that
	\begin{equation}
		\lim_{N\to\infty}\lambda_{*}(\mathcal{M}_{u}(\mu_{N}))=\lambda_{*}(\mathcal{M}_{u}(\mu))\label{eq:lambdastarconv}
	\end{equation}
	in probability. Together, \cref{eq:lambda1conv} and \cref{eq:lambdastarconv}
	complete the proof of the theorem.
\end{proof}

\begin{proof}[Proof of \cref{t.stoch.ball}]
	We set $\lambda_{\mathrm{c}} \coloneqq \lambda_1(\mu)$. Using \cref{thm:stability} with $u = 0$, we see that $\lambda_1(\mu_N)$ tends to $\lambda_{\mathrm{c}}$ in probability as $N$ tends to infinity. Part \eqref{part1} of \cref{t.stoch.ball} thus follows.

	We now turn to the proof of part \eqref{part2}, and fix $\lambda > \lambda_{\mathrm{c}}$. By the definition of $\lambda_{\mathrm{c}}$ and \cref{thm:agglomeration}, the range of $u_{\mu,\lambda}$ contains at least two points. We decompose the rest of the proof into two steps.

	\emph{Step 1.} We show that the range of $u_{\mu,\lambda}$ contains at least three points. We argue by contradiction, assuming that the range of $u_{\mu,\lambda}$ is made of exactly two points. Notice that the measure $\mu$ is symmetric under rotations about the first coordinate axis, and under negations of any of the canonical basis vectors. By the uniqueness of the minimizer, it must be that $u_{\mu,\lambda}$ is invariant under these transformations. As we now argue, the range of $u_{\mu,\lambda}$ must therefore be a subset of the first coordinate axis. Indeed, this is easiest to see if $d\ge 3$, since otherwise the range of $u_{\mu,\lambda}$ would have to contain a circle, and in particular would contain infinitely many points. Suppose now that $d=2$ and that the range of $u_{\mu,\lambda}$ is made of exactly two points. By the invariance under reflections, the only possibility for the support to not be a subset of the first coordinate axis is that the two points forming the support of $u_{\mu,\lambda}$ are on the second coordinate axis; but in this case, the two level sets of $u_{\mu,\lambda}$ would each have to contain half of each of the balls, and this would contradict \cref{prop:ballsdontintersect}. 
	
Using again the invariance under reflections, we deduce that there exists $\rho > 0$ such that the range of $u_{\mu,\lambda}$ is the set  $\{-\rho \e_1, \rho \e_1\}$. Let $E \coloneqq u_{\mu,\lambda}^{-1}(\rho \e_1)$. Again by symmetry, it must be that, up to a set of null $\mu$-measure, we have $u_{\mu,\lambda}^{-1}(-\rho \e_1) = -E$, and $\mu(E) = \mu(-E) = 1/2$, so that
	\begin{equation}
		\label{e.constant.fusion}
		\iint |u_{\mu,\lambda}(x) - u_{\mu,\lambda}(y)| \, \d \mu(x) \, \d \mu(y) = \rho.
	\end{equation}
	Moreover,
	\begin{align*}  %
		\int_E |\rho \e_1 - x|^2 \, \d \mu(x)
		 & = \int_{E \cap B_1(re_1)} |\rho \e_1 - x|^2 \, \d \mu(x) + \int_{E \cap B_1(-re_1)} |\rho \e_1 - x|^2 \, \d \mu(x)
		\\
		 & = \int_{E \cap B_1(re_1)} |\rho \e_1 - x|^2 \, \d \mu(x) + \int_{(-E) \cap B_1(re_1)} |\rho \e_1 + x|^2 \, \d \mu(x)
		\\
		 & \ge \int_{E \cap B_1(re_1)} |\rho \e_1 - x|^2 \, \d \mu(x) + \int_{(-E) \cap B_1(re_1)} |\rho \e_1 - x|^2 \, \d \mu(x)
		\\
		 & \ge \int_{B_1(re_1)} |\rho \e_1 - x|^2 \, \d \mu(x),
	\end{align*}
	since $E \cap (-E)$ is a $\mu$-null set. This yields that
	\begin{equation*}  %
		\int |u_{\mu,\lambda}-x|^2 \, \d \mu(x) \ge \int_{B_1(re_1)} |\rho \e_1 - x|^2 \, \d \mu(x) + \int_{B_1(-re_1)} |-\rho \e_1 - x|^2 \, \d \mu(x).
	\end{equation*}
	Combining this with \eqref{e.constant.fusion}, we see that we must have, up to a $\mu$-null set, that $E = B_1(r\e_1)$. In other words, the minimizer $u_{\mu,\lambda}$ maps $B_1(r \e_1)$ to $\rho \e_1$ and $B_1(-r \e_1)$ to $-\rho \e_1$.

	By \cref{thm:splitcondition}, we must therefore have that
	\begin{equation}
		\label{e.cohesive.part}
		\mbox{the measure $\frac 1 2 \delta_{-r \e_1} + \frac 1 2 \delta_{r \e_1}$ is $\lambda$-shattered},
	\end{equation}
	and
	\begin{equation}
		\label{e.shattered.part}
		\mbox{the measure $\mu|_{B_1(r\e_1)}$ is $\lambda$-cohesive}.
	\end{equation}
	By \cref{prop:twopoints}, the requirement in \eqref{e.cohesive.part} imposes that $\lambda \le 2r$. By \cref{p.ball}, the requirement in \eqref{e.shattered.part} imposes that $\lambda \ge 2 \gamma_d$. Since we assume that $r < \gamma_d$, we have reached a contradiction.

	\emph{Step 2.} By the result of the previous step, there exist $c_1,c_2,c_3 \in \Rd$ and $\eta > 0$ such that for every $i\neq j \in \{1,2,3\}$, we have $|c_i - c_j| \ge 9 \eta$, and
	\begin{equation}
		\label{e.min.mass}
		\xi \coloneqq \min\Ll(\mu[u_{\mu,\lambda}^{-1}(B_\eta(c_1))], \mu[u_{\mu,\lambda}^{-1}(B_\eta(c_2))], \mu[u_{\mu,\lambda}^{-1}(B_\eta(c_3))]\Rr) > 0.
	\end{equation}
	Since the measure $\mu$ is absolutely continuous with respect to the Lebesgue measure, there exists a $1$-optimal transport map from $\mu$ to $\mu_N$, which we denote by $T_N$. By \cref{prop:W1-perturbations}, we have
	\begin{equation*}  %
		\int \Ll| u_{\mu,\lambda}(x) - u_{\mu_N,\lambda}(T_N(x)) \Rr| \, \d \mu(x)\le 16 M \mathcal{W}_1(\mu,\mu_N).
	\end{equation*}
In particular, for each $i \in \{1,2,3\}$, we have
	\begin{equation*}  %
		\int_{u_{\mu,\lambda}^{-1}(B_\eta(c_i))} |c_i - u_{\mu_N,\lambda}(T_N(x))|  \, \d \mu(x) \le 16 M \mathcal{W}_1(\mu,\mu_N) + \eta\mu[u_{\mu,\lambda}^{-1}(B_\eta(c_i))].
	\end{equation*}
		Recall that $\mathcal{W}_1(\mu,\mu_N)$ tends to zero in probability as $N$ tends to infinity \citep[see for instance][]{dudley}. 
	For every $\ep > 0$, we can therefore let $N$ be sufficiently large that with probability at least $1-\ep$, we have
	\begin{equation*}  %
		\int_{u_{\mu,\lambda}^{-1}(B_\eta(c_i))} |c_i - u_{\mu_N,\lambda}(T_N(x))|  \, \d \mu(x) \le 2\eta\mu[u_{\mu,\lambda}^{-1}(B_\eta(c_i))].
	\end{equation*}
	In particular, by Chebyshev's inequality,
	\begin{equation*}  %
		\int_{u_{\mu,\lambda}^{-1}(B_\eta(c_i))} \1_{\{|c_i - u_{\mu_N,\lambda}(T_N(x))|\ge 4\eta\}}  \, \d \mu(x) \le \frac 1 2 \mu[u_{\mu,\lambda}^{-1}(B_\eta(c_i))];
	\end{equation*}
	that is,
	\begin{equation*}  %
		\int_{u_{\mu,\lambda}^{-1}(B_\eta(c_i))} \1_{\{|c_i - u_{\mu_N,\lambda}(T_N(x))|< 4\eta\}}  \, \d \mu(x) \ge \frac 1 2 \mu[u_{\mu,\lambda}^{-1}(B_\eta(c_i))].
	\end{equation*}
	Recalling that $T_N$ is an optimal transport map from $\mu$ to $\mu_N$, we see that the left side is bounded from above by
	\begin{equation*}  %
		\int \1_{\{|c_i - u_{\mu_N,\lambda}(x)| < 4 \eta\}} \, \d \mu_N(x) = \frac 1 N \Ll| \{n \le N \ : \ |c_i - u_{\mu_N,\lambda}(X_n) | < 4 \eta\} \Rr|.
	\end{equation*}
	Recalling also the definition of $\xi$, we have shown that, with probability at least $1-\ep$, the following holds for every $N$ sufficiently large and $i \in \{1,2,3\}$:
	\begin{equation*}  %
		\frac 1 N \Ll| \{n \le N \ : \ |c_i - u_{\mu_N,\lambda}(X_n) | < 4 \eta\} \Rr| \ge \frac \xi 2.
	\end{equation*}
	Since $|c_i - c_j| \ge 9 \eta$ for every $i \neq j$, this yields the desired result, up to a redefinition of $\xi$.
\end{proof}

To conclude, we give a counterpart to \cref{t.stoch.ball} in the case when the two balls are sufficiently far apart.

\begin{prop}
	\label{p.separation}
	Let $r  > 2^{1-\frac 1 d}$, $\mu$ be the uniform measure on $B_1(-r\e_1) \cup B_1(r\e_1) \subset \R^d$, $(X_n)_{n \in \N}$ be independent random variables with law $\mu$, and for every integer $N \ge 1$, define the empirical measure
	\begin{equation*}
		\mu_N \coloneqq \frac 1 N \sum_{n = 1}^N \delta_{X_n}.
	\end{equation*}
	If $\lambda \in (2^{2-\frac 1 d},2r)$, then with high probability, the level sets of $u_{\mu_N,\lambda}$ are the sets
	\begin{equation*}  %
		\{X_n, \ n \le N\} \cap B_1(-r \e_1) \quad \text{ and } \quad \{X_n, \ n \le N\} \cap B_1(r \e_1).
	\end{equation*}
\end{prop}
\begin{proof}
	By \cref{thm:splitcondition}, the level sets of the function $u_{\mu,\lambda}$ are, up to $\mu$-null modifications, the two balls $B_1(-r\e_1)$ and $B_1(r\e_1)$, if and only if \eqref{e.cohesive.part} and \eqref{e.shattered.part} hold. By \cref{prop:twopoints}, the first condition holds whenever $\lambda < 2r$, and by \cref{p.ball}, the second condition holds whenever $\lambda > 2\cdot 2^{1-\frac 1 d}$. The result then follows by an application of \cref{thm:stability}.
\end{proof}

\section{Technical lemmas}
In this section we collect a few additional technical lemmas to avoid distracting from the flow of the paper.
\begin{lem}\label{lem:Adoesnotmatter}
Let $\mu$ be a finite Borel measure on $\R^d$.
    Let $u_1,u_2\colon\R^d\to\R^d$ be such that $u_1(x) = u_2(x)$ for $\mu$-a.e.~$x$, and let $A_1,A_2 \subseteq \R^d$ be Borel sets such that, for each $i=1,2$, we have $\mu(\R^d\setminus A_i)=0$ and $V_{u_i,x}\cap A_i$ is $\mu$-regular for $\mu$-a.e.~$x$. If we define $\mathcal{E}^{(i)}(x)\coloneqq \mathcal{C}_{\mu}(V_{u_i,x}\cap A_i)$, then $\mathcal{E}^{(1)}(x) = \mathcal{E}^{(2)}(x)$ for $\mu$-a.e.~$x$.
\end{lem}
\begin{proof}
    Let $B$ be the set of all $x\in\R^d$ such that $u_1(x) = u_2(x)$. Note that $\mu(A_1\cap A_2\cap B) = \mu(\R^d)$.
    Let $x\in A_1\cap A_2\cap B$. We claim that $\mathcal{E}^{(1)}(x) = \mathcal{E}^{(2)}(x)$. 
    We consider two cases.

    First, suppose that there is some $i$ such that $\mu(V_{u_i,x})>0$, and assume wlog that $i=1$. Then we have $V_{u_1,x}\cap B = \{y\in B\ :\ u_1(x) = u_1(y)\}= \{y\in B\ :\ u_2(x) = u_2(x)\} = V_{u_2,x}\cap B$, since $u_1(z) = u_2(z)$ for all $z\in B$. Since $\mu(\R^d\setminus B) =0$, this implies that 
    $\mathcal{C}_\mu(V_{u_i,x}\cap A_i)$ does not depend on $i$, since changing a positive-measure set by a set of measure zero does not change its centroid. %

    On the other hand, if $x$ is such that $\mu(V_{u_1,x}) = \mu(V_{u_2,x}) = 0$, then $V_{u_i,x}\cap A_i = \{x\}$ for each $i$ by $\mu$-regularity, and hence $\mathcal{C}_\mu(V_{u,x}\cap A_i) = x$ for each $i$.

    Thus we have shown that the set of $x$ such that $\mathcal{E}^{(1)}(x) \ne \mathcal{E}^{(2)}(x)$ is contained in $\R^d\setminus(A_1\cap A_2\cap B)$, which has $\mu$-measure $0$. %
\end{proof}

\begin{lem}\label{lem:Jcts}
    For any finite Borel measure $\mu$ and any $\lambda\ge 0$, the function $J_{\mu,\lambda}\colon L^2(\mu;\R^d)\to\R$ defined in \cref{e.def.J} is continuous.
\end{lem}
\begin{proof}
    Let $u_1,u_2\in L^2(\mu;\R^d)$. We have by the triangle, reverse triangle, and Cauchy--Schwarz inequalities that
    \begin{align*}
        &\left|\iint |u_1(x)-u_1(y)|\,\dif \mu(x)\,\dif \mu(y)-\iint |u_2(x)-u_2(y)|\,\dif \mu(x)\,\dif \mu(y)\right|\\
        &\qquad\le\iint (|u_1(x)-u_2(x)|+|u_1(y)-u_2(y)|)\,\dif \mu(x)\,\dif \mu(y)\le 2\mu(\R^d)^{3/2}\|u_1-u_2\|_{L^2(\mu;\R^d)}.
    \end{align*}
    Similarly, we have
    \begin{align*}
        &\left|\int|u_1(x)-x|^2\,\dif \mu(x)-\int|u_2(x)-x|^2\,\dif \mu(x)\right|\\&\qquad\le 2\int |u_1(x)-u_2(x)|^2\,\dif \mu(x)\le 2\|u_1-u_2\|_{L^2(\mu;\R^d)}^2.
    \end{align*}
    Together, the last two displays imply that $J_{\mu,\lambda}$ is continuous.
\end{proof}

\subsection*{Acknowledgments}
We warmly thank Antonio De Rosa for many interesting discussions. AD~was partially supported by the NSF Mathematical Sciences Postdoctoral Fellowship program under grant no.\ DMS-2002118. JCM~was partially supported by the NSF grant DMS-1954357.

\bibliography{convex-clustering}

\end{document}